\documentclass[11pt,reqno]{amsart}
\usepackage[margin=1in]{geometry}
\usepackage[foot]{amsaddr}
\usepackage{longtable}

\usepackage{graphicx}
\usepackage{comment}
\usepackage{multirow}

\usepackage{times}
\usepackage{comment}
\usepackage{enumerate}

\usepackage{amsmath, amssymb, url, color, hyperref, bbm, natbib}
\hypersetup{
    colorlinks,
    citecolor=blue,
    filecolor=black,
    linkcolor=black,
    urlcolor=black
}

\newcommand{\R}{\mathbb{R}}

\renewcommand{\P}{\mathbb{P}}
\newcommand{\E}{\mathbb{E}}

\newcommand{\btheta}{\boldsymbol{\theta}}
\newcommand{\x}{\mathbf{x}}
\newcommand{\X}{\mathbf{X}}
\newcommand{\y}{\mathbf{y}}

\newcommand{\g}{\mathbf{g}}

\newcommand{\e}{\mathbf{e}}
\renewcommand{\u}{\mathbf{u}}
\renewcommand{\v}{\mathbf{v}}
\newcommand{\w}{\mathbf{w}}

\newcommand{\bP}{\mathbf{P}}
\newcommand{\cL}{\mathcal{L}}
\newcommand{\cP}{\mathcal{P}}
\newcommand{\cZ}{\mathcal{Z}}
\newcommand{\Q}{\mathbf{Q}}
\newcommand{\T}{\mathsf{T}}
\newcommand{\N}{\mathcal{N}}

\newcommand{\1}{\mathbf{1}}
\newcommand{\I}{\mathbf{I}}

\newcommand{\op}{\operatorname{op}}

\newcommand{\Cov}{\mathrm{Cov}}
\newcommand{\Var}{\mathrm{Var}}

\DeclareMathOperator*{\argmin}{arg\,min}

\newcommand{\bSigma}{\mathbf{\Sigma}}

\newtheorem{theorem}{Theorem}[section]
\newtheorem{corollary}[theorem]{Corollary}
\newtheorem{lemma}[theorem]{Lemma}
\newtheorem{proposition}[theorem]{Proposition}
\theoremstyle{definition}
\newtheorem{definition}[theorem]{Definition}

\begin{document}

\title[Tree-projected gradient descent]{Tree-projected gradient descent for
estimating gradient-sparse parameters on graphs}

\author{Sheng Xu}
\email{sheng.xu@yale.edu}

\author{Zhou Fan}
\email{zhou.fan@yale.edu}

\author{Sahand Negahban}
\email{sahand.negahban@yale.edu}
\address{Department of Statistics and Data Science \\ Yale University \\
New Haven, CT 06511}

\begin{abstract}%
We study estimation of a gradient-sparse parameter vector $\btheta^* \in \R^p$,
having strong gradient-sparsity $s^*:=\|\nabla_G \btheta^*\|_0$ on
an underlying graph $G$. Given observations
$Z_1,\ldots,Z_n$ and a smooth, convex loss function $\cL$ for which
$\btheta^*$ minimizes the population risk $\E[\cL(\btheta;Z_1,\ldots,Z_n)]$,
we propose to estimate $\btheta^*$
by a projected gradient descent algorithm that iteratively and approximately projects gradient steps onto spaces of vectors
having small gradient-sparsity over low-degree spanning trees
of $G$. We show that, under suitable restricted strong
convexity and smoothness assumptions for the loss, the resulting estimator
achieves the squared-error risk $\frac{s^*}{n} \log (1+\frac{p}{s^*})$
up to a multiplicative constant that is independent of $G$.
In contrast, previous polynomial-time algorithms have only been shown to
achieve this guarantee in more specialized settings, or
under additional assumptions for $G$
and/or the sparsity pattern of $\nabla_G \btheta^*$. As applications of our
general framework, we apply our results to the examples of
linear models and generalized linear models with random design.
\end{abstract}

\maketitle


\section{Introduction}

We study estimation of a piecewise-constant or gradient-sparse
parameter vector on a given graph. This problem may arise in statistical changepoint
detection \citep{killick2012optimal,fryzlewicz2014wild}, where an unknown vector on a line graph has a sequential changepoint structure. In image denoising
\citep{rudin1992nonlinear} and compressed sensing
\citep{candes2006robust,donoho2006compressed}, this vector may represent a gradient-sparse image on a 2D or 3D lattice graph, as arising in medical X-rays and CT scans. For applications of epidemic tracking and anomaly detection on general graphs and networks, this vector may indicate regions of infected or abnormal nodes \citep{ariascastro2011detection}.

We consider the following general framework: Given observations $Z_1^n:=(Z_1,\ldots,Z_n) \in \cZ^n$
with distribution $\cP$, we seek to estimate a parameter
$\btheta^* \in \R^p$ associated
to $\cP$. The coordinates of $\btheta^*$ are identified with the vertices of a
known graph $G=(V,E)$, where the number of vertices is $|V|=p$.
Denoting by $\nabla_G:\R^p \to \R^{|E|}$ the discrete gradient operator
\begin{equation}\label{eq:graphgrad}
\nabla_G \btheta=\big(\theta_i-\theta_j:(i,j) \in E\big),
\end{equation}
we assume that the gradient sparsity $s^*:=\|\nabla_G \btheta^*\|_0$ is
small relative to the total number of edges in $G$.
For example, when $G$ is a line or lattice graph, $s^*$ measures the number of
changepoints or the total boundary size between the constant pieces of an image, respectively.
For a given convex and differentiable loss function $\cL:\R^p \times \cZ^n
\to \R$, we assume that $\btheta^*$ is related to
the data distribution $\cP$ as the minimizer of the population risk,
\[\btheta^*=\argmin_{\btheta \in \R^p}
\E_\cP\big[\cL(\btheta;Z_1^n)\big].\]
Important examples include linear and generalized linear models
for $Z_i=(\x_i,y_i)$, where $\btheta^*$ is the vector of regression coefficients
and $\cL$ is the usual squared-error or negative log-likelihood loss.

Our main result implies that, under suitable restricted strong convexity and
smoothness properties of the loss \citep{negahban2012unified}
and subgaussian assumptions on the noise, a
polynomial-time projected gradient descent algorithm yields an estimate
$\widehat{\btheta}$ which achieves the squared-error guarantee
\begin{equation}\label{eq:mainguarantee}
\|\widehat{\btheta}-\btheta^*\|_2^2 \leq
C \cdot \frac{s^*}{n}\log\Big(1+\frac{p}{s^*}\Big)
\end{equation}
with high probability. Here, $C>0$ is a constant independent of the graph $G$, and depends only on the loss $\cL$ and distribution $\cP$ via their convexity, smoothness, and subgaussian constants.

Despite the simplicity of the guarantee (\ref{eq:mainguarantee}) and its
similarity to results for estimating \emph{coordinate-sparse} parameters
$\btheta^* \in \R^p$, to our knowledge, our work is the first to establish
this guarantee in polynomial time for estimating \emph{gradient-sparse} parameters on general
graphs, including the 1D line. In particular, (\ref{eq:mainguarantee}) is not
necessarily achieved by convex approaches which constrain or regularize the
$\ell_1$ (total-variation) relaxation $\|\nabla_G \btheta^*\|_1$,
for the reason that an ill-conditioned discrete gradient matrix $\nabla_G \in \R^{|E| \times p}$
contributes to the restricted convexity and smoothness properties of the resulting convex problem \citep{hutter2016optimal,fan2018approximate}. We
discuss this further below, in the context of related literature.

Our work instead analyzes an algorithm that
iteratively and approximately computes the projected gradient update
\begin{equation}\label{eq:approxproj}
\btheta_t \approx \argmin_{\btheta \in \R^p:\|\nabla_{T_t} \btheta\|_0 \leq S}
\|\btheta-\btheta_{t-1}+\eta \cdot \nabla \cL(\btheta_{t-1};Z_1^n)\|_2
\end{equation}
over a sequence of low-degree spanning trees $T_1,T_2,\ldots$ of $G$.\footnote{Here, $\nabla \cL(\btheta_{t-1};Z_1^n)$ is the gradient of $\cL(\btheta;Z_1^n)$ with respect to $\btheta$ at $\btheta_{t-1}$, and $\nabla_{T_t} \btheta$ is the discrete gradient operator
(\ref{eq:graphgrad}) over the edges in $T_t$ instead of $G$.}
To obtain a polynomial-time algorithm,
we approximate each projection onto the non-convex space
$\{\btheta \in \R^p:\|\nabla_{T_t} \btheta\|_0 \leq S\}$
by discretizing the signal domain $\R^p$ and applying a
dynamic-programming recursion over $T_t$ to compute the discrete projection. For graphs $G$ that do not admit spanning trees of
low degree, we apply an idea of \citep{padilla2017dfs} and
construct $T_t$ using a combination of edges in $G$ and additional edges
representing backtracking paths along a depth-first-search traversal of $G$.

Our algorithm and analysis rely on an important insight from
\citep{jain2014iterative}, which is to perform each projection using a target
sparsity-level $S$ that is larger than the true gradient-sparsity $s^*$
by a constant factor. This idea was applied
in \citep{jain2014iterative} to provide a statistical analysis of
iterative thresholding procedures such as IHT, CoSaMP, and HTP for estimating
coordinate-sparse parameters
\citep{blumensath2009iterative,needell2009cosamp,foucart2011hard}.
A key ingredient in our proof, Lemma \ref{lem:key} below, is a combinatorial
argument which compares the errors of approximating any vector $\u$ by vectors $\u^S$ and $\u^*$ that are gradient-sparse over a tree, with two different sparsity levels $S$ and $s^*$. This extends a central lemma of \citep{jain2014iterative} from the simpler setting of coordinate-sparsity to a setting of gradient-sparsity on trees.

\subsection{Related literature}\hfill

\noindent
Existing literature on this and related problems is extensive, and we provide
here a necessarily partial overview.

{\bf Convex approaches: } Estimating a piecewise-constant vector $\btheta^*$ in both the
direct-measurements model $y_i=\theta^*_i+e_i$ and the indirect
linear model $y_i=\x_i^\top \btheta^*+e_i$
has been of interest since early work on the fused lasso
\citep{tibshirani2005sparsity,rinaldo2009properties} and compressed sensing
\citep{candes2006stable,candes2006robust,donoho2006compressed}.
A natural and commonly-used approach is to constrain or penalize the
total-variation semi-norm $\|\nabla_G \btheta^*\|_1$
\citep{rudin1992nonlinear}. Statistical properties of this approach have been
extensively studied, including estimation guarantees over signal classes of 
either bounded variation or bounded exact gradient-sparsity
\citep{mammen1997locally,hutter2016optimal,sadhanala2016total,dalalyan2017prediction,lin2017sharp,ortelli2018total};
exact or robust recovery guarantees in compressed sensing contexts
\citep{needell2013near,needell2013stable,cai2015guarantees}; and correct
identification of changepoints or of the discrete gradient support
\citep{harchaoui2010multiple,sharpnack2012sparsistency}. Extensions to
higher-order trend-filtering methods have been proposed and studied in
\citep{kim2009ell_1,wang2016trend,sadhanala2017higher,guntuboyina2017adaptive}.
These works have collectively considered settings of both direct and indirect
linear measurements, for the 1D line, 2D and 3D lattices, and more
general graphs.

 In the above work, statistical guarantees analogous to
(\ref{eq:mainguarantee}) have only been obtained under
restrictions for either $G$ or $\btheta^*$, which we are able to remove using a non-convex approach.
\citep{hutter2016optimal} established  a guarantee analogous to (\ref{eq:mainguarantee}) when certain
compatibility and inverse-scaling factors of $G$ are $O(1)$; a
sufficient condition is that $G$ has constant maximum degree, and the
Moore-Penrose pseudo-inverse $\nabla_G^\dagger$ has constant
$\ell_1 \to \ell_2$ operator norm. This notably does not include the 1D line or
2D lattice. \citep{dalalyan2017prediction},
\citep{lin2017sharp}, and \citep{guntuboyina2017adaptive} developed
complementary results, showing that (\ref{eq:mainguarantee}) can hold
for the 1D line provided that the $s^*$ changepoints of $\btheta^*$ have
minimum spacing $\gtrsim p/(s^*+1)$. An extension of this to
tree graphs was proven in \citep{ortelli2018total}. Roughly speaking,
$\nabla_G^\dagger$ is an effective design matrix for an associated sparse
regression problem, and the spacing condition ensures that the \emph{active}
variables in the regression model are weakly correlated, even if the
full design $\nabla_G^\dagger$ has strong correlations.

{\bf Synthesis approach: } A separate line of work focuses on the \emph{synthesis} approach, which
uses a sparse representation of $\btheta^*$ in an orthonormal basis
or more general dictionary. Such methods include wavelet approaches
in 1D
\citep{daubechies1988orthonormal,donoho1994ideal,donoho1995adapting},
curvelet and ridgelet frames in 2D
\citep{candes1998ridgelets,candes2000curvelets,candes2004new},
and tree-based wavelets for more general graphs
\citep{gavish2010multiscale,sharpnack2013detecting}.
 \citep{elad2007analysis} and \citep{nam2013cosparse} compare and discuss differences between
the synthesis and analysis approaches. Note that in general, an $s^*$-gradient-sparse signal $\btheta^*$ may not
admit a $O(s^*)$-sparse representation in an orthonormal basis. For example, $\btheta^*$ having $s^*$ changepoints on the line may
have up to $s^*\log_2 p$ non-zero coefficients in the Haar wavelet basis, and (\ref{eq:mainguarantee}) would be inflated by an additional log factor using Haar wavelets.

{\bf Our contributions: }
In contrast to this first line of work on convex methods, our current work is
most closely related to a third line of literature on methods that penalize or
constrain the exact non-convex gradient-sparsity $\|\nabla_G
\btheta^*\|_0$, rather than its convex $\ell_1$ relaxation
\citep{mumford1989optimal,boykov2001fast,boysen2009consistencies,fan2018approximate}. 
 This direct method
enables theoretical guarantees that remove the spectral conditions on the graph
$G$ as well as the minimum spacing requirements of the work alluded to above.

Our results extend those of \citep{fan2018approximate}, which established similar
guarantees to (\ref{eq:mainguarantee}) for
direct measurements $y_i=\theta^*_i+e_i$. Our projected gradient algorithm is
similar to the proximal-gradient method recently studied in
\citep{xu2019iterative}, which considered indirect linear measurements
$y_i=\x_i^\top \btheta^*+e_i$ in a compressed sensing context. In contrast to
\citep{xu2019iterative}, which considered deterministic measurement errors and a
restrictive RIP-type condition on the measurement design, we provide guarantees
in the statistical setting of random noise, with much weaker conditions for the
regression design, and for a general convex loss. These statistical guarantees are based on a novel
tree-projection algorithm that approximates the graph at every iteration. The
analysis leverages a new bound that controls the approximation error of tree
projections, which is presented in Lemma~\ref{lem:key}.

\section{Tree-projected gradient descent algorithm}
Our proposed algorithm, tree-projected gradient descent (tree-PGD), consists of two main steps:
\begin{enumerate}
\item For a specified vertex degree $d_{\max} \geq 2$ and iteration count $\tau \geq 1$, we construct a sequence of trees $T_1,\ldots,T_\tau$ on the same vertices as $G$, such that each tree $T_t$ has maximum degree $\leq d_{\max}$, and any gradient-sparse vector on $G$ remains gradient-sparse on $T_t$.
\item For a specified step size $\eta>0$ and sparsity level $S>0$, we compute iterates $\btheta_1,\ldots,\btheta_\tau$ where each $\btheta_t$ solves the projected gradient-descent step (\ref{eq:approxproj}) over a discretized domain---see (\ref{eq:discretization}) and (\ref{eq:proj}) below.
\end{enumerate}

For simplicity, we initialize the algorithm at $\btheta_0=0$. The main tuning parameter is the projection sparsity $S$, which controls the bias-variance trade-off and the gradient sparsity of the final estimate $\widehat{\btheta}=\btheta_\tau$.
The additional parameters of the algorithm are $d_{\max}$, $\tau$, $\eta$, and the discretization (\ref{eq:discretization})
specified by $(\Delta_{\min},\Delta_{\max},\delta)$. We discuss these two steps in detail below.

For our theoretical guarantees, it is sufficient to choose $d_{\max}=2$ and to fix the same tree in every iteration. However, we observe in Section \ref{sec:simulations} that using both larger values of $d_{\max}$ and a different random tree in each iteration can yield substantially lower recovery error in practice, so we will state our algorithm and theory to allow for these possibilities.

\subsection{Tree construction}\label{sec:treeconstruction}

We construct a tree $T$ on the vertices $V=\{1,\ldots,p\}$ by the following procedure.

\begin{enumerate}
\item Compute any spanning tree $\tilde{T}$ of $G$. If $\tilde{T}$ has maximum degree $\leq d_{\max}$, then set $T=\tilde{T}$.
\item Otherwise, let $\mathcal{O}_{DFS}$ be the ordering of unique vertices and edges visited in any depth-first-search (DFS) traversal of $\tilde{T}$. For each vertex $v$ whose degree exceeds $d_{\max}$ in $\tilde{T}$, keep its first $d_{\max}$ edges in this ordering, and delete its remaining edges from $\tilde{T}$. Note that the deleted edges are between $v$ and its children.
\item For each such deleted edge $(v,w)$ where $w$ is a child of $v$, let $w'$ be the vertex preceding $w$ in the ordering $\mathcal{O}_{DFS}$, and add to $\tilde{T}$ the edge $(w',w)$. Let $T$ be the final tree.
\end{enumerate}
This procedure is illustrated in Figure \ref{fig:treeconstruction}. We repeat this construction to obtain each tree $T_1,\ldots,T_\tau$.

\begin{figure}
\begin{center}
\includegraphics[width=0.3\textwidth]{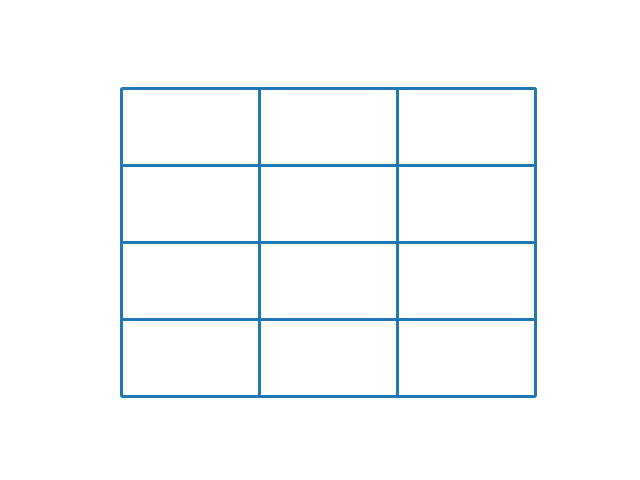}
\includegraphics[width=0.3\textwidth]{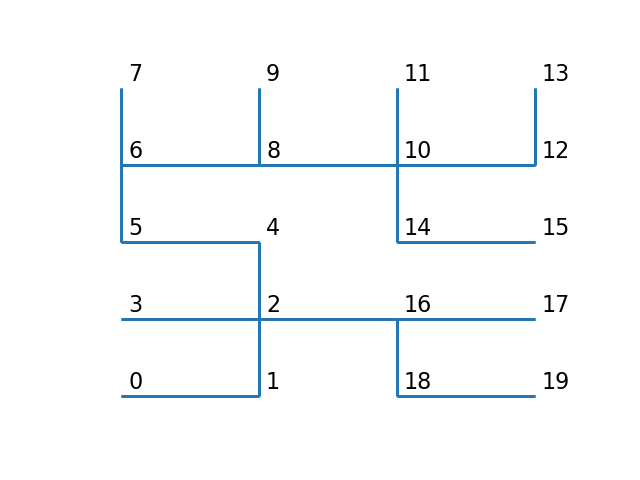}
\includegraphics[width=0.3\textwidth]{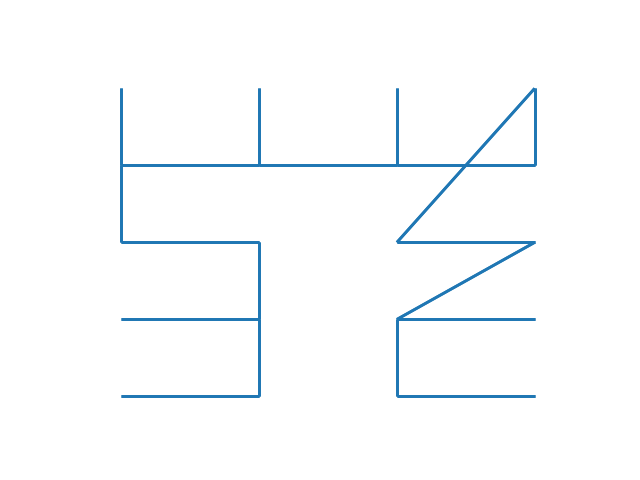}
\end{center}
\vspace{-1cm}
\caption{An illustration of the tree construction method. Left: Original lattice
graph $G$. Middle: A spanning tree $\tilde{T}$ of $G$, with vertices numbered in
DFS ordering. Right: The final tree $T$ with
$d_{\max}=3$, which changes edge $(2,16)$ to $(15,16)$, and edge $(10,14)$ to
$(13,14)$, thus replacing the two edges adjacent to the degree-4 vertices of
$T$.}\label{fig:treeconstruction}
\end{figure}

If $G$ itself has maximum degree $\leq d_{\max}$, then Steps 2 and 3 above are not necessary, and the guarantee (\ref{eq:tree}) below may be trivially strengthened to $\|\nabla_T \btheta\|_0 \leq \|\nabla_G \btheta\|_0$. For graphs $G$ of larger maximum degree, the idea in Steps 2 and 3 above and the associated guarantee (\ref{eq:tree}) are drawn from Lemma 1 of \citep{padilla2017dfs}, which considered the case of a line graph for $T$ (where $d_{\max}=2$).

\begin{lemma}\label{lem:tree}
Let $G=(V,E)$ be any connected graph with $p$ vertices, and
let $T$ be as constructed above. Then $T$ is a tree on $V$ with maximum degree $\leq d_{\max}$. Furthermore, for any $\btheta \in \R^p$,
\begin{equation}\label{eq:tree}
\|\nabla_T \btheta\|_0 \leq 2\|\nabla_G \btheta\|_0.
\end{equation}
The computational complexity for constructing $T$ is $O(|E|)$.
\end{lemma}

\subsection{Projected gradient approximation}

The exact minimizer of (\ref{eq:approxproj})
is the projection of $\u_t:=\btheta_{t-1}-\eta \cdot \nabla \cL(\btheta_{t-1};Z_1^n)$ onto the space of $S$-gradient-sparse vectors over $T_t$. This space is a union of $\binom{p-1}{S}$ linear subspaces, and naively iterating over these subspaces is intractable for large $S$. We instead propose to approximate the projection by taking a discrete grid of values
\begin{equation}\label{eq:discretization}
\Delta:=\big\{\Delta_{\min},\Delta_{\min}+\delta,\Delta_{\min}+2\delta,\ldots,\Delta_{\max}-\delta,\Delta_{\max}\big\}
\end{equation}
and performing the minimization over $\btheta \in \Delta^p$.
Thus, our tree-PGD algorithm sets
\begin{equation}\label{eq:proj}
\btheta_t=\argmin_{\btheta \in \Delta^p:\|\nabla_{T_t}\btheta\|_0 \leq S}
\|\btheta-\btheta_{t-1}+\eta \cdot \nabla \cL(\btheta_{t-1};Z_1^n)\|_2
\end{equation}
Each $\btheta_t$ may be computed by a dynamic-programming recursion over $T_t$.\footnote{For the case where $T_t$ is a line graph, an alternative non-discretized algorithm with complexity $O(p^2S)$ is presented in \citep{auger1989algorithms}.}

In detail, fix any target vector $\u \in \R^p$ and a tree $T$ on the vertices $\{1,\ldots,p\}$. To compute
\begin{equation}\label{eq:dpobj}
\argmin_{\btheta \in \Delta^p:\|\nabla_T \btheta\|_0 \leq S} \|\btheta-\u\|_2,
\end{equation}
pick any vertex $o \in \{1,\ldots,p\}$ with degree 1 in $T$ as the root. For each vertex $v$ of $T$, let $T_v$ be the sub-tree consisting of $v$ and its descendants. Let $|T_v|$ be the number of vertices in $T_v$ and $\u_{T_v} \in \R^{|T_v|}$ be the coordinates of $\u$ belonging to $T_v$. Define $f_v:\Delta \times \{0,1,\ldots,S\} \to \R$ by
\begin{equation}\label{eq:dp}
f_v(c,s)=\min\Big\{\|\btheta-\u_{T_v}\|_2^2:\btheta \in \Delta^{|T_v|},\,\|\nabla_{T_v} \btheta\|_0 \leq s,\,\theta_v=c\Big\}.
\end{equation}
This is the minimum over vectors $\btheta$ on $T_v$ that are $s$-gradient-sparse and take value $c \in \Delta$ at $v$.
These values $f_v(c,s)$ may be computed recursively from the leaves to the root, as follows.

\begin{enumerate}
\item For each leaf vertex $v$ of $T$ and each $(c,s) \in \Delta \times \{0,1,\ldots,S\}$, set $f_v(c,s)=(c-u_v)^2$.
\item For each vertex $v$ of $T$ with children $(w_1,\ldots,w_k)$, given $f_w(c,s)$ for all $w \in \{w_1,\ldots,w_k\}$ and $(c,s) \in \Delta \times \{0,1,\ldots,S\}$:
\begin{enumerate}
\item For each $s \in \{0,1,\ldots,S\}$ and $w \in \{w_1,\ldots,w_k\}$, compute
$m_w(s)=\min_{c \in \Delta} f_w(c,s)$.
\item For each $(c,s) \in \Delta \times \{0,1,\ldots,S\}$ and $w \in \{w_1,\ldots,w_k\}$, compute $g_w(c,s)=\min\{f_w(c,s),m_w(s-1)\}$,
where this is taken to be $f_w(c,s)$ if $s=0$.
\item For each $(c,s) \in \Delta \times \{0,1,\ldots,S\}$,  set
\begin{equation}\label{eq:dp2}
f_v(c,s)=(c-u_v)^2+\mathop{\min_{s_1,\ldots,s_k \geq 0}}_{s_1+\ldots+s_k=s}
\Big(g_{w_1}(c,s_1)+\ldots+g_{w_k}(c,s_k)\Big).
\end{equation}
\end{enumerate}
\end{enumerate}
The following then produces the vector $\btheta$ which solves (\ref{eq:dpobj}). \begin{enumerate}
\item[3.] For the root vertex $o$, set $\theta_o=\argmin_{c \in \Delta} f_o(c,S)$ and $S_o=S$.
\item[4.] For each other vertex $v$, given $\theta_v$ and $S_v$: Let $w_1,\ldots,w_k$ be the children of $v$ and let $s_1,\ldots,s_k$ be the choices which minimized (\ref{eq:dp2}) for $f_v(\theta_v,S_v)$. For each $i=1,\ldots,k$, if $g_{w_i}(\theta_v,s_i)=f_{w_i}(\theta_v,s_i)$, then set $\theta_{w_i}=\theta_v$ and $S_{w_i}=s_i$. If $g_{w_i}(\theta_v,s_i)=m_{w_i}(s_i-1)$, then set $\theta_{w_i}=\argmin_{c \in \Delta} f_{w_i}(c,s_i-1)$ and $S_{w_i}=s_i-1$.
\end{enumerate}
The update $\btheta_t$ in (\ref{eq:proj}) is computed by applying this algorithm to $\u \equiv \u_t=\btheta_{t-1}-\eta \cdot \nabla \cL(\btheta_{t-1};Z_1^n)$.

\begin{lemma}\label{lem:dp}
This algorithm minimizes (\ref{eq:dpobj}). Letting $d_{\max}$ be the maximum vertex degree of $T$ and $|\Delta|$ be the cardinality of $\Delta$, its computational complexity is $O(d_{\max}p|\Delta|(S+d_{\max})^{d_{\max}-1})$.
\end{lemma}

\subsection{Total complexity for the linear model}

Let us compute the total complexity of this tree-PGD algorithm, under parameter settings that yield a rate-optimal statistical guarantee for the linear model discussed in Section \ref{sec:linearmodel}.
We set $d_{\max}$ as a small integer and $S$ as a constant multiple of $s^*$. Evaluating $\nabla \cL(\btheta_{t-1};Z_1^n)$ in the linear model requires two matrix-vector multiplications of complexity $O(np)$, where $n$ is the sample size. Let us assume that the number of graph edges is $|E|=O(p)$, and that the entries of $\btheta^*$ and the noise $\e$ are both of constant order. Then Corollary \ref{cor:prop1} indicates that we may take 
$\Delta_{\max}-\Delta_{\min}=O(\sqrt{p})$, $\delta=O(\sqrt{s^*/np})$, and $\tau=O(\log np)$.
Under these settings, the total complexity of tree-PGD is $O\Big(\big(np+p^2\sqrt{n}(s^*)^{d_{\max}-3/2}\big)\log np\Big)$.
Setting $d_{\max}=2$ (i.e.\ taking  $T_1,\ldots,T_\tau$ to be line graphs) yields the lowest complexity.

\section{Main theorem}

We introduce the following notation which identifies gradient-sparse vectors, partitions of the vertices $\{1,\ldots,p\}$, and subspaces of $\R^p$.

\begin{definition}
Let $T$ be a connected graph on the vertices $V=\{1,\ldots,p\}$, and let $\btheta \in \R^p$. The {\bf partition induced by $\btheta$ over $T$} is the partition of $V$ whose sets are the connected components of $\{(i,j) \in T:\theta_i=\theta_j\}$ in $T$. For such a partition $\cP$ having $k$ sets, the {\bf subspace associated to $\cP$} is the dimension-$k$ subspace of vectors in $\R^p$ taking a constant value over each set. The {\bf boundary of $\cP$ over $T$}, denoted by $\partial_T \cP$, is the set of edges $(i,j) \in T$ where $i,j$ belong to different sets of $\cP$.
\end{definition}

Thus, the sets of the partition $\cP$ induced by $\btheta$ over $T$ are the ``pieces'' of the graph $T$ where $\btheta$ takes a constant value.
If $\cP$ is induced by $\btheta$ over $T$, and $K$ is the associated subspace, then $\btheta \in K$. Furthermore, $\partial_T \cP$ is exactly the edge set where $\nabla_T \btheta$ is non-zero, and $\|\nabla_T \btheta\|_0=|\partial_T \cP|$.

We introduce two properties for the loss, defined for pairs of connected graphs $(T_1,T_2)$ on the same vertices $V$. We will apply these to consecutive pairs of trees generated by tree-PGD.

\begin{definition}[cRSC and cRSS]\label{def:cRSC_cRSS} 
A differentiable function $f: \R^p\to \R$ satisfies cut-restricted strong convexity (cRSC) and smoothness (cRSS) with respect to $(T_1,T_2)$, at sparsity level $S$ and with convexity and smoothness constants $\alpha,L>0$, if the following holds: For any partitions $\cP_1,\cP_2$ of $\{1,\ldots,p\}$ where $|\partial_{T_1} \cP_1| \leq S$
and $|\partial_{T_2} \cP_2| \leq S$, and any $\btheta_1,\btheta_2 \in K:=K_1+K_2$ where $K_1,K_2$ are the subspaces associated to $\cP_1,\cP_2$,
\begin{align}
f(\btheta_2)\geq f(\btheta_1)+\langle\btheta_2-\btheta_1,\nabla f(\btheta_1)\rangle+\frac{\alpha}{2}\|\btheta_2-\btheta_1\|_2^2,\label{def:cRSC}\\
f(\btheta_2)\leq f(\btheta_1)+ \langle\btheta_2-\btheta_1,\nabla f(\btheta_1)\rangle+\frac{L}{2}\|\btheta_2-\btheta_1\|_2^2.\label{def:cRSS}
\end{align}
\end{definition}

\begin{definition}[cPGB]\label{def:ProjGrad}
A differentiable function $f: \R^p\to \R$ has a cut-projected gradient bound
(cPGB) of $\Phi(S)$ with respect to $(T_1,T_2)$, at a point $\btheta^* \in \R^p$
and sparsity level $S$, if the following holds: For any partitions $\cP_1,\cP_2$
of $\{1,\ldots,p\}$ where $|\partial_{T_1}\cP_1|\leq S$ and
$|\partial_{T_2}\cP_2|\leq S$, letting $K_1,K_2$ be their associated subspaces and
$\bP_K$ be the orthogonal projection onto $K:=K_1+K_2$,
\begin{align}\label{eq:ProjGrad}
\|\bP_K \nabla f(\btheta^*)\|_2 \leq \Phi(S).
\end{align}
\end{definition}

To provide some interpretation, the below lemma gives an example for this function $\Phi$ in the important setting where $\w^{\T}\nabla \cL(\btheta^*;Z_1^n)$ is subgaussian for any $\w\in K$.

\begin{lemma}\label{lem:ProjGrad}
Let $S \geq 1$, let $T_1,T_2$ be trees on $\{1,\ldots,p\}$, and let $\btheta^* \in \R^p$. Suppose, for any subspace $K$ as defined in Definition \ref{def:ProjGrad} and any $\w\in K$, that $\w^\top \nabla\cL(\btheta^*;Z_1^n)$ is $\sigma^2/n$-subgaussian.\footnote{This means that for any $t>0$, $\P[|\w^\top \nabla\cL(\btheta^*;Z_1^n)|>t] \leq 2e^{-nt^2/(2\sigma^2)}$.}
Then for any $k>0$ and a constant $C_k>0$ depending only on $k$, with probability at least $1-p^{-k}$, the loss $\cL(\cdot\,;Z_1^n)$ has the cPGB
\[\Phi(S)=C_k\sigma\sqrt{\tfrac{S}{n}\log\left(1+\tfrac{p}{S}\right) }\]
with respect to $(T_1,T_2)$, at $\btheta^*$ and sparsity level $S$.
\end{lemma}

The following is our main result, which provides a deterministic estimation guarantee when tree-PGB is applied with an appropriate choice of the projection sparsity $S=\kappa s^*$. This result yields the same type of guarantee for any choice of $d_{\max} \geq 2$ and any sequence of trees.

\begin{theorem}\label{thm:main}
Suppose $\|\nabla_G\btheta^*\|_0 \leq s^*$, where $s^*>0$.
Set $S=\kappa s^*$ in tree-PGD for a constant $\kappa>1$. Let $\tau \geq 1$ and $d_{\max} \geq 2$,
let $T_1,\ldots,T_{\tau}$ be the sequence of trees generated by tree-PGD, and denote $T_0=T_1$ and $S'=S+2s^*+\max(\sqrt{S},d_{\max})$.
Suppose, for all $1\leq t\leq \tau$, that
\begin{enumerate}
\item $\cL(\cdot\,;Z_1^n)$ satisfies cRSC and cRSS with respect to $(T_{t-1},T_t)$, at sparsity level $S'$ and with convexity and smoothness constants $\alpha,L>0$.
\item $\cL(\cdot\,;Z_1^n)$ has the cPGB $\Phi(S')$
with respect to $(T_{t-1},T_t)$, at the point $\btheta^*$ and sparsity level $S'$.
\end{enumerate}
Define
\[\gamma=\sqrt{\tfrac{(d_{\max}-1)(2s^*+\sqrt{S}+1)+1}
{S-2s^*-\sqrt{S}}}, \quad \Gamma=(1+\gamma)\sqrt{1-\tfrac{\alpha}{L}}, \quad \Lambda=\tfrac{1}{1-\Gamma}\left(\tfrac{4(1+\gamma)}{\alpha}\cdot\Phi(S')+\delta\sqrt{p}\right),\]
and suppose $\kappa$ is large enough such that $S>\sqrt{S}+2s^*$ and $\Gamma<1$.
Take $\eta=\frac{1}{L}$, $\btheta_0=0$, and $-\Delta_{\min},\Delta_{\max} \geq \frac{1}{L}\|\nabla \cL(\btheta^*;Z_1^n)\|_\infty+3\|\btheta^*\|_2+2\Lambda$ in tree-PGD. Then the $\tau^\text{th}$ iterate $\btheta_\tau$ of tree-PGD satisfies
\begin{align*}
\|\btheta_{\tau}-\btheta^*\|_2 \leq \Gamma^{\tau}\cdot\|\btheta^*\|_2+\Lambda.
\end{align*}
\end{theorem}

Note that since $\gamma \to 0$ as $\kappa \to \infty$, for any value $\alpha/L \in (0,1]$, there is a choice of constant $\kappa \equiv \kappa(\alpha,L)$ sufficiently large to ensure $\Gamma<1$.

\subsection{Proof overview}

The proof of Theorem \ref{thm:main} adopts an induction argument. For simplicity, let us suppose here that $\btheta_t$ exactly minimizes (\ref{eq:approxproj}). Then for each iteration, we wish to prove
\begin{align}\label{eq:main_idea}
\|\btheta_t-\btheta^*\|_2\leq\Gamma\cdot\|\btheta_{t-1}-\btheta^*\|_2+\tfrac{4(1+\gamma)}{\alpha}\cdot\Phi(S').
\end{align}

The proof of (\ref{eq:main_idea}) contains two main steps. First, we construct a subspace $K$ which contains $\btheta_t$ and $\btheta^*$ and write $\|\btheta_t-\btheta^*\|_2 \leq \|\bP_K\u_t-\btheta_t\|_2+\|\bP_K\u_t-\btheta^*\|_2$. Using the following key lemma, we show that there exists such a subspace $K$ for which $\|\bP_K\u_t-\btheta_t\|_2 \leq \gamma \|\bP_K\u_t-\btheta^*\|_2$, and the vectors in $K$ have gradient-sparsity not much larger than $S+s^*$.

\begin{lemma}\label{lem:key}
Let $T$ be a tree on the vertices $\{1,\ldots,p\}$ with maximum vertex degree
$d_{\max}$. Let $s^*>0$ and $S=\kappa s^*$,
where $\kappa>1$ and $S>\sqrt{S}+s^*$. Let $\u \in \R^p$ be arbitrary, let $\u^* \in \R^p$ be any vector satisfying $\|\nabla_T \u^*\|_0 \leq s^*$, and set
\[\u^S=\argmin_{\btheta \in \R^p:\|\nabla_T \btheta\|_0 \leq S} \|\u-\btheta\|_2.\]
Denote by $(K^S,K^*)$ the subspaces associated to the partitions induced by
$(\u^S,\u^*)$ over $T$.
Then there exists a partition $\cP$ of $\{1,\ldots,p\}$ with associated subspace
$K$, such that $K$ contains $K^S+K^*$,
\begin{equation}\label{eq:Psparsity}
|\partial_T \cP| \leq S+s^*+\sqrt{S},
\end{equation}
and the orthogonal projection $\bP_K\u$ of $\u$ onto $K$ satisfies
\begin{equation}\label{eq:projcomparison}
\|\bP_K\u-\u^S\|_2^2 \leq \frac{(d_{\max}-1)(s^*+\sqrt{S}+1)+1}
{S-s^*-\sqrt{S}}\,\|\bP_K\u-\u^*\|_2^2.
\end{equation}
\end{lemma}

Then, in the second step, we bound $\|\bP_K\u_t-\btheta^*\|_2$ by introducing $\v=\argmin_{\btheta \in K} \cL(\btheta;Z_1^n)$. Using a property of the gradient mapping (Lemma \ref{lem:main_step2}) and the cRSC and cRSS conditions, we show that $\|\bP_K\u_t-\v\|_2 \leq \sqrt{1-\alpha/L} \cdot \|\btheta_{t-1}-\v\|_2$. Applying the triangle inequality, this implies $\|\bP_K\u_t-\btheta^*\|_2 \leq \sqrt{1-\alpha/L} \cdot \|\btheta_{t-1}-\btheta^*\|_2+2\|\v-\btheta^*\|_2$. Finally, we show that $\|\v-\btheta^*\|_2 \leq (2/\alpha)\Phi(S')$ using the cRSC and cPGB properties of the loss, and combining gives (\ref{eq:main_idea}).

The use of Lemma \ref{lem:key} is inspired by an analogous argument of \citep{jain2014iterative} for coordinate-sparse parameter estimation. However, the analysis for coordinate-sparsity is simpler, due to a key structural property that if $\u^S$ and $\u^*$ are the best (coordinate-) $S$-sparse and $s^*$-sparse approximations of $\u$, then the sparse subspace of $\u^*$ is contained inside that of $\u^S$. This nested subspace structure does not hold for gradient-sparsity, and thus our proofs of both Lemma \ref{lem:key} and Theorem \ref{thm:main} follow different arguments from those of \citep{jain2014iterative}.

\section{Examples}\label{sec:examples}
\subsection{Gradient-Sparse Linear Regression}\label{sec:linearmodel}
Consider the example of $Z_i=(\x_i,y_i)$ satisfying a linear model
\begin{equation}\label{eq:linearmodel}
y_i=\x_i^\top \btheta^*+e_i
\end{equation}
for independent design vectors $\x_i \in \R^p$ and mean-zero residual errors $e_i$. Let us write this as $\y=\X\btheta^*+\e$
where $\y=(y_1,\ldots,y_n)$, $\e=(e_1,\ldots,e_n)$, and $\X \in \R^{n \times p}$ is the random design matrix with rows $\x_i^\top$. Then $\btheta^*$ is the minimizer of $\E[\cL(\btheta;Z_1^n)]$ for the squared-error loss
\[\cL(\btheta;Z_1^n)=\frac{1}{2n}\|\y-\X\btheta\|_2^2.\]
The gradient of the loss is given by $\nabla\cL(\btheta;Z_1^n)=\X^\top(\X\btheta-\y)/n$.

We assume that
\begin{equation}\label{eq:xconditions}
\Cov(\x_i)=\bSigma, 
\qquad \lambda_{\max}(\bSigma)=\lambda_1, \qquad \lambda_{\min}(\bSigma)=\lambda_p,
\qquad \|\x_i\|_{\psi_2}^2 \leq D \lambda_p
\end{equation}
\begin{equation}\label{eq:econditions}
\E[e_i]=0, \qquad \|e_i\|_{\psi_2}^2 \leq \sigma^2
\end{equation}
for constants $\lambda_1,\lambda_p,D,\sigma^2>0$, where $\|\cdot\|_{\psi_2}$ denotes the scalar or vector subgaussian norm. Then the cRSC, cRSS, and cPGB conditions hold according to the following proposition.

\begin{proposition}\label{thm:prop1}
Suppose (\ref{eq:xconditions}) and (\ref{eq:econditions}) hold, and let $S' \geq 1$. Define
\begin{equation}\label{eq:gsS}
g(S')=S'\log(1+\tfrac{p}{S'}).
\end{equation}
Let $T_1,\ldots,T_\tau$ be the trees generated by tree-PGD, and let $T_0=T_1$. For any $k>0$, and some constants $C_1,C_2,C_3>0$ depending only on $k$ and $D$, if
\[n \geq C_1g(S')\]
then with probability at least $1-\tau\cdot p^{-k}$, for every $1 \leq t \leq \tau$,
\begin{enumerate}
\item $\cL(\cdot\,;Z_1^n)$ satisfies cRSC and cRSS with respect to $(T_{t-1},T_t)$ at sparsity level $S'$ and with convexity and smoothness constants $\alpha=\lambda_p/2$ and $L=3\lambda_1/2$.
\item $\cL(\cdot\,;Z_1^n)$ has the cPGB
\[\Phi(S')=C_2\sigma\sqrt{\lambda_1 g(S')/n}\]
with respect to $(T_{t-1},T_t)$, at $\btheta^*$ and sparsity level $S'$.
\item $\|\nabla \cL(\btheta^*;Z_1^n)\|_\infty \leq C_3\sigma\sqrt{(\lambda_1 \log p)/n}$.
\end{enumerate}
\end{proposition}

Applying this and Theorem \ref{thm:main}, we obtain the following immediate corollary.

\begin{corollary}\label{cor:prop1}
Suppose (\ref{eq:xconditions}) and (\ref{eq:econditions}) hold, and $\|\nabla_G \btheta^*\|_0 \leq s^*$ and $\|\btheta^*\|_2 \leq c_0\sqrt{p}$ for some $s^* \geq 1$ and $c_0>0$. Set $S=c_1(\lambda_1/\lambda_p)^2s^*$, $\eta=2/(3\lambda_1)$, $\omega=\sigma\lambda_1^{3/2}/\lambda_p^2$, $-\Delta_{\min}=\Delta_{\max}=c_2(\sqrt{p}+\omega\sqrt{(s^*\log p)/n})$, $\delta=\omega\sqrt{s^*/np}$, and $\tau=c_3\log(np/\omega^2s^*)$ in tree-PGD, for sufficiently large constants $c_1>0$ depending on $d_{\max},D$ and $c_2,c_3>0$ depending on $d_{\max},D,c_0$.

Then for any $k>0$ and some constants $C_1,C_2>0$ depending only on $k,d_{\max},D$, if $n \geq C_1(\lambda_1/\lambda_p)^2s^*\log(1+p/s^*)$, then with probability at least $1-\tau \cdot p^{-k}$,
\begin{align*}
\|\btheta_\tau-\btheta^*\|_2^2 \leq C_2 \cdot \frac{\sigma^2\lambda_1^3}{\lambda_p^4} \cdot \frac{s^*}{n}\log\Big(1+\frac{p}{s^*}\Big).
\end{align*}
\end{corollary}

\subsection{Gradient-Sparse GLM}
Consider the example of $Z_i=(\x_i,y_i)$ satisfying a generalized linear model (GLM)
\begin{align*}
P(y_i|\x_i,\btheta^*,\phi)=\exp\Big\{ \frac{y_i \x_i^{\T}\btheta^*-b(\x_i^{\T}\btheta^*)}{\phi}\Big\}\cdot h(y_i,\phi)
\end{align*}
for independent design vectors $\x_i \in \R^p$. Here $\phi>0$ is a constant scale parameter, and $h$ and $b$ are the base measure and cumulant function of the exponential family, where $\E(y_i|\x_i)=b'(\x_i^\top \btheta^*)$. Then $\btheta^*$ minimizes the population risk $\E[\cL(\btheta;Z_1^n)]$ for the negative log-likelihood loss
\[\cL(\btheta;Z_1^n)=\frac{1}{n}\sum_{i=1}^n \Big(b(\x_i^\top\btheta)-y_i\x_i^\top \btheta\Big).\]
The gradient of this loss is $\nabla\cL(\btheta;Z_1^n)=\frac{1}{n}\sum_{i=1}^n (b'(\x_i^\top\btheta)-y_i)\x_i$.

Let us assume that (\ref{eq:xconditions}) holds for the design vectors $\x_i$. Setting $e_i=y_i-b'(\x_i^\top \btheta^*)$, let us assume also that for some constants $\alpha_b,L_b,D_1,D_2>0$ and $\beta \in [1,2]$,
\begin{equation}\label{eq:GLMbcondition}
\frac{\alpha_{b}}{2}(x_2-x_1)^2\leq b(x_2)-b(x_1)-b'(x_1)(x_2-x_1)\leq\frac{L_{b}}{2}(x_2-x_1)^2 \quad \text{ for all } x_1,x_2 \in \R,
\end{equation}
\begin{equation}\label{eq:GLMecondition}
\P(|e_i|>\zeta)\leq D_1\exp(-D_2 \zeta^\beta) \quad \text{ for all } \zeta>0.
\end{equation}
Then the cRSC, cRSS, and cPGB conditions hold according to the following proposition.

\begin{proposition}\label{thm:prop2}
Suppose that (\ref{eq:xconditions}), (\ref{eq:GLMbcondition}), and (\ref{eq:GLMecondition}) hold. Let $S' \geq 1$ and $g(S')$ be as in (\ref{eq:gsS}). Let $T_1,\ldots,T_\tau$ be the trees generated by tree-PGD, and let $T_0=T_1$. For any $k>0$ and some constants $C_1,C_2,C_3>0$ depending only on $k,D,D_1,D_2,\beta$, if $n \geq C_1g(S')$, then with probability at least $1-\tau\cdot p^{-k}$, for every $1 \leq t \leq \tau$,
\begin{enumerate}
\item $\cL(\cdot\,;Z_1^n)$ satisfies cRSC and cRSS with respect to $(T_{t-1},T_t)$ at sparsity levels $S'$ with convexity and smoothness constants $\alpha=\frac{\alpha_b\lambda_p}{2}$ and $L=\frac{3L_b\lambda_1}{2}$.
\item $\cL(\cdot\,;Z_1^n)$ has the cPGB 
\begin{align*}
\Phi(S')=\begin{cases}
C_2\sqrt{\lambda_1/n} \cdot g(S')^{1/\beta} & \text{ if } \quad 1<\beta \leq 2\\
C_2\log n\sqrt{\lambda_1/n} \cdot g(S') & \text{ if } \quad \beta=1\end{cases}
\end{align*}
with respect to $(T_{t-1},T_t)$, at $\btheta^*$ and sparsity level $S'$.
\item
$\displaystyle \|\nabla \cL(\btheta^*;Z_1^n)\|_\infty \leq
\begin{cases}
C_3(\log p)^{1/\beta}\sqrt{\lambda_1/n} & \text{ if } \quad 1<\beta \leq 2\\
C_3(\log n)(\log p)\sqrt{\lambda_1/n} & \text{ if } \quad \beta=1 \end{cases}$
\end{enumerate}
\end{proposition}

Under suitable settings of the tree-PGD parameters, similar to Corollary \ref{cor:prop1} and which we omit for brevity, when $n \geq C's^*\log(1+p/s^*)$, this yields the estimation rate
\[\|\btheta_\tau-\btheta^*\|_2^2 \leq C \cdot \frac{(s^* \log(1+p/s^*))^{2/\beta}}{n}\]
in models where $1<\beta \leq 2$, and this rate with an additional $(\log n)^2$ factor in models where $\beta=1$. (Here, these constants $C,C'$ depend on $\lambda_1,\lambda_p,D,D_1,D_2,\beta$.)

We note that this result may be established under a relaxed condition (\ref{eq:GLMbcondition}) that only holds over a sufficiently large bounded region for $x_1,x_2$, following a more delicate analysis and ideas of \citep{negahban2012unified}. For simplicity, we will not pursue this direction in this work.

\section{Simulations}\label{sec:simulations}

\begin{figure}
\begin{center}
\includegraphics[width=0.2\textwidth]{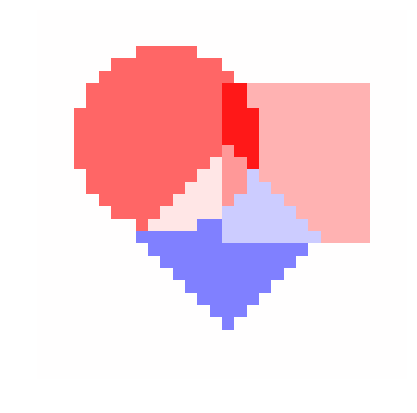}
\includegraphics[width=0.2\textwidth]{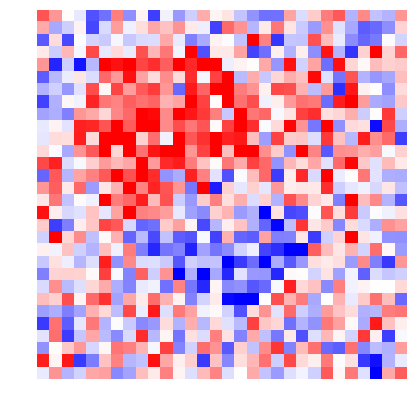}
\includegraphics[width=0.2\textwidth]{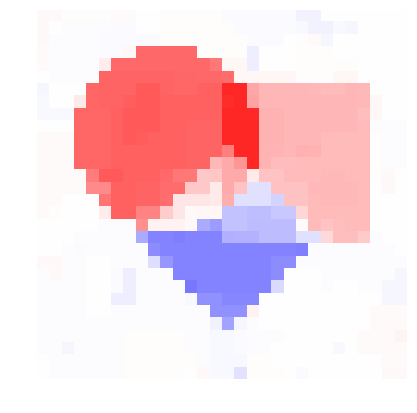}\\
\includegraphics[width=0.2\textwidth]{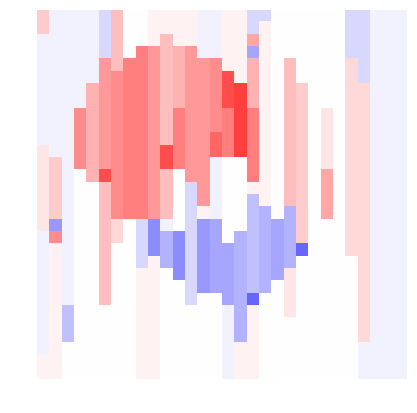}
\includegraphics[width=0.2\textwidth]{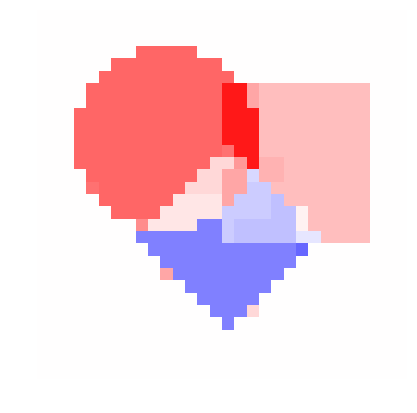}
\includegraphics[width=0.2\textwidth]{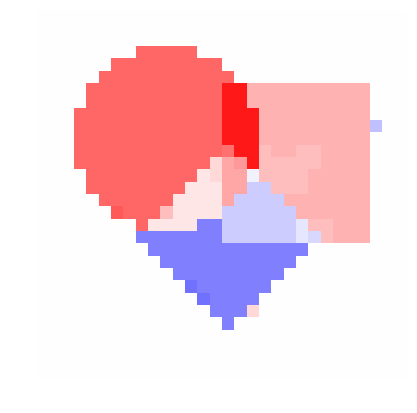}
\end{center}
\caption{Top-left: True image $\btheta^*$, with values between $-0.5$ (blue) and $0.9$ (red). Top-middle: Noisy image $\frac{1}{n}\X^\top \y$, for $\y=\X\btheta^*+\e$ with Gaussian design and noise standard deviation $\sigma=1.5$. Top-right: Best total-variation penalized estimate $\widehat{\btheta}$. Bottom row: Best tree-PGD estimate $\widehat{\btheta}$ for a fixed line graph $T_t$ in every iteration (zig-zagging vertically through $G$, bottom left), a different random tree with $d_{\max}=2$ in each iteration (bottom middle), and a different random tree with $d_{\max}=4$ in each iteration (bottom right).}\label{fig:2D}
\end{figure}

\begin{table}
\begin{center}
\begin{tabular}{|l||c|c|c|c|c|}
\hline
Noise std. dev. $\sigma$ & 1.0 & 1.5 & 2.0 & 2.5 & 3.0 \\
\hline
Fixed line & 0.0372 & 0.0373 & 0.0383 & 0.0388 & 0.0407 \\
Random, $d_{\max}=2$ & 0.0005 & 0.0009 & 0.0020 & 0.0040 & 0.0058 \\
Random, $d_{\max}=3$ & 0.0003 & 0.0008 & 0.0014 & 0.0028 & 0.0052 \\
Random, $d_{\max}=4$ & 0.0003 & 0.0007 & 0.0013 & 0.0032 & 0.0055 \\
\hline
Total variation & 0.0006 & 0.0013 & 0.0023 & 0.0036 & 0.0052 \\
\hline
\end{tabular}
\end{center}
\caption{MSE $\frac{1}{p}\|\widehat{\btheta}-\btheta^*\|_2^2$ for recovering the image of Figure \ref{fig:2D} (under best tuning of $S$), averaged across 20 independent simulations. For tree-PGD, using a different random tree $T_t$ per iteration yields a sizeable improvement over using a fixed line graph across all iterations, and small improvements are observed for increasing $d_{\max}$. Average MSE for the total-variation penalized estimate is provided for comparison (under best tuning of $\lambda$).}\label{tab:2D}
\end{table}

Theorem \ref{thm:main} applies for any choices of trees $T_1,\ldots,T_\tau$ in tree-PGD, with any maximum degree $d_{\max} \geq 2$. We perform a small simulation study in the linear model (\ref{eq:linearmodel}) to compare the empirical estimation accuracy of tree-PGD using different tree constructions.

We recover the image $\btheta^*$ depicted in Figure \ref{fig:2D} on a $30 \times 30$ lattice graph $G$, using $n=500$ linear measurements with $\x_i \sim \N(0,\I)$ and $e_i \sim \N(0,\sigma^2)$. For $\sigma=1.5$, a noisy image $\frac{1}{n}\X^\top \y=\btheta^*+(\frac{1}{n}\X^\top \X-\I)\btheta^*+\frac{1}{n}\X^\top\e$ is also depicted.

\paragraph{Tree construction: }We applied tree-PGD in two settings: First, we constructed $T_t$ using a deterministic DFS over $G$, fixed across all iterations. This resulted in $T_t$ being a line graph that zig-zags vertically through $G$. Second, we constructed $T_t$ using a different spanning tree $\tilde{T}_t$ generated by random DFS in each iteration. The DFS procedure started at a uniform random node and, at each forward step, chose a uniform random unvisited neighbor. We tested restricting to $d_{\max}=2$ or $d_{\max}=3$ for $T_t$, or letting $T_t=\tilde{T}_t$ (corresponding to $d_{\max}=4$). In all experiments, we used $\tau=80$, $\eta=1/5$, and $(\Delta_{\min},\Delta_{\max},\delta)=(-0.6,1.0,0.05)$.

Results for a single experiment at $\sigma=1.5$ are depicted in Figure \ref{fig:2D}, and average MSE across 20 experiments for varying $\sigma$ are reported in Table \ref{tab:2D}. These results correspond to the best choices $S=\kappa s^*$ across a range of tested values. Estimation accuracy is substantially better using different and random trees than using the same fixed line graph. We observe small improvements using $d_{\max}=3$ or $d_{\max}=4$ over random line graphs with $d_{\max}=2$, especially in the higher signal-to-noise settings. For comparison, we display in Figure \ref{fig:2D} and Table \ref{tab:2D} also the total-variation (TV) regularized estimate $\widehat{\btheta}=\argmin_{\btheta} \frac{1}{2n}\|\y-\X\btheta\|_2^2+\lambda \|\nabla_G \btheta\|_1$ and its average MSE, corresponding to the best choices of $\lambda$. We observe that tree-PGD, which targets the exact gradient-sparsity rather than a convex surrogate, is more accurate in high signal-to-noise settings, and becomes less accurate in comparison with TV as signal strength decreases. This agrees with previous observations made in similar contexts in \citep{hastie2017extended,mazumder2017subset,fan2018approximate}.

\section{Discussion}
\label{sec:discussion} We have shown linear convergence of gradient descent with
projections onto the non-convex space of gradient-sparse vectors on a graph. Our results
show that this method achieves strong statistical guarantees in regression models, without requiring a
matching between the underlying graph and design matrix. We do this by introducing
a careful comparison between gradient-sparse approximations at different sparsity levels, which generalizes
previous results for coordinate-sparse vectors.

Our theory is presented in such a way that allows the approximation trees to
vary at each iteration. However, this is not required and the tree
can be fixed with $d_{\max} = 2$ at the start of the algorithm. Nevertheless, we
observe experimentally that using a different random tree in each iteration
substantially improves the practical performance. Our intuition for the
improvement with random trees is that the gradient-sparsity of the signal on the
original graph $G$ may be better captured by the average sparsity with respect to
a randomly chosen sub-tree of $G$, than by the sparsity with respect to any fixed
sub-tree. By using a different random tree in each iteration, the algorithm is
better targeting this average sparsity. This observation will be studied in future
work.

Another interesting direction for future work is to explore the connections between this work and
computationally tractable sparse linear regression problems with highly
correlated designs. For instance, some work \citep{buhlmann2013correlated,
  dalalyan2017prediction} discuss various ways to overcome correlated designs. In our setting,
the tree projection step enables a computationally efficient method, and it is of interest to understand more general settings where one may overcome the correlated structure of the problem using a computationally efficient procedure.

\section{Acknowledgement}

This research is supported in part by NSF Grant DMS-1916198 and DMS-1723128.

\newpage

\bibliography{treePGD}

\appendix

\section{Correctness and complexity of algorithm}

We prove Lemmas \ref{lem:tree} and \ref{lem:dp} on basic guarantees for the two steps of the tree-PGD algorithm.\\

\begin{proof}[Proof of Lemma \ref{lem:tree}]
For the first statement, since $d_{\max} \geq 2$, the vertex $w$ corresponding to each deleted edge $(v,w)$ must be a child of $v$ which is not its first child in the ordering $\mathcal{O}_{DFS}$. Then its preceding vertex $w'$ must be a leaf vertex of $\tilde{T}$. Each such $w$ corresponds to a different such leaf $w'$, so deleting these edges $(v,w)$ and adding $(w',w)$ preserves the connectedness and tree structure. By construction, each non-leaf vertex of $\tilde{T}$ has degree at most $d_{\max}$ in $T$. Each leaf vertex of $\tilde{T}$ has degree at most $2 \leq d_{\max}$ in $T$, so $T$ has maximum degree $\leq d_{\max}$.

For the second statement, since the edges of $\tilde{T}$ are a subset of those of $G$,
\begin{align*}
\|\nabla_{\tilde{T}} \btheta\|_0\leq \|\nabla_G \btheta\|_0.
\end{align*}
Let the root vertex of $T$ be 1. For each other vertex $i \geq 2$, denote its parent in $T$ by $p(i)$. Then
\begin{align}\label{eq:tree1}
\|\nabla_T\btheta\|_0=\sum_{i=2}^{p}\1\{\theta_i \neq \theta_{p(i)}\}.
\end{align}
Now consider two cases: If the edge $(i,p(i))$ exists in $\tilde{T}$, then it is a forward edge in the DFS of $\tilde{T}$, and $\1\{\theta_i \neq \theta_{p(i)}\}$ contributes to $\|\nabla_{\tilde{T}} \btheta\|_0$. If $(i,p(i))$ is not an edge of $\tilde{T}$, then $p(i)$ is a leaf node in $\tilde{T}$, and there is path of backward edges $(p_1,p_2,\ldots,p_r)$ in the DFS of $\tilde{T}$ where $p_1=p(i)$ and $p_r=i$. The triangle inequality then implies 
\begin{align*}
\1\{\theta_i \neq \theta_{p(i)}\} \leq 
\sum_{j=1}^{r-1}\1\{\theta_{p_j} \neq \theta_{p_{j+1}}\},
\end{align*}
where each term on the right contributes to $\|\nabla_{\tilde{T}} \btheta\|_0$.
Applying this to each term
on the right of \eqref{eq:tree1}, and invoking the fundamental property that DFS visits each edge of $\tilde{T}$ exactly twice, we get
\begin{align*}
\|\nabla_T \btheta\|_0\leq 2\|\nabla_{\tilde{T}} \btheta\|_0\leq 2\|\nabla_G \btheta\|_0.
\end{align*}
\end{proof}

\begin{proof}[Proof of Lemma \ref{lem:dp}]
It is clear that Step 1 computes (\ref{eq:dp}) at the leaf vertices $v$. For Step 2, assume inductively that $f_w(c,s)$ is the value (\ref{eq:dp}) for all children $w$ of $v$. The value $g_w(c,s)$ represents the minimum value of $\|\btheta-\u_{T_w}\|_2^2$, if $\theta_v=c$ and the gradient-sparsity of $\btheta$ on $T_w$ \emph{and} the additional edge $(v,w)$ is at most $s$---we have either $\theta_w=c$ and $g_w(c,s)=f_w(c,s)$, or $\theta_w \neq c$, in which case $\theta_w=\argmin_{c \in \Delta} f_w(c,s-1)$ and $g_w(c,s)=m_w(s-1)$.
Then (\ref{eq:dp2}) computes (\ref{eq:dp}) at $v$ by partitioning the gradient-sparsity $s$ across its $k$ children, and summing the costs $g_{w_i}(c,s_i)$ and the additional cost $(c-u_v)^2$ for the best such partition. Thus Step 2 correctly computes (\ref{eq:dp}) for each vertex $v$. In particular, the minimum value for (\ref{eq:dpobj}) is given by $\min_{c \in \Delta} f_o(c,S)$. The minimizer $\btheta$ is obtained by examining the minimizing choices in Steps 1 and 2, which is carried out in Steps 3 and 4: Each $\theta_v$ is the value of $\btheta$ at $v$, and each $S_v$ is (an upper-bound for) the value of $\|\nabla_{T_v} \btheta\|_0$ at the minimizer $\btheta$.

For each vertex $v$, Step 1 has complexity $(S+1)|\Delta|$, Steps 2(a) and 2(b) both have complexity $(S+1)k|\Delta|$, and Step 2(c) has complexity $(S+1)|\Delta|k\binom{S+k-1}{k-1}$, as there are $\binom{s+k-1}{k-1} \leq \binom{S+k-1}{k-1}$ partitions of $s$ into $s_1,\ldots,s_k$. Note that $k \leq d_{\max}-1$, where this holds also for the root vertex $o$ because we chose it to have degree 1 in $T$. Then $\binom{S+k-1}{k-1}=O((S+d_{\max})^{d_{\max}-2})$. Storing the relevant minimizers in Steps 1 and 2, the complexity of Steps 3 and 4 is $O(1)$ per vertex. So the total complexity is $O(d_{\max}p|\Delta|(S+d_{\max})^{d_{\max}-1})$.
\end{proof}

\section{Proof of Lemma \ref{lem:key}}

\begin{proof}
Let $\cP^S$ be the partition of $\{1,\ldots,p\}$ induced by $\u^S$ over $T$.
We have $|\partial_T \cP^S| \leq S$. If $|\partial_T \cP^S|<S$, then let us
arbitrarily split some vertex sets in $\cP^S$ along edges of $T$,
until $|\partial_T \cP^S|=S$. Thus, we may assume
henceforth that $|\partial_T \cP^S|=S$.

We construct another partition $\cP'$ of $\{1,\ldots,p\}$
into the (disjoint) vertex sets $(V_1,\ldots,V_B,R)$, such that each set
of $\cP'$ is connected over $T$, and $\cP'$ satisfies the following properties:
\begin{enumerate}
\item For each $b=1,\ldots,B$, the number of edges $(i,j)$ in $T$ where both
$i,j \in V_b$, but $i$ and $j$ do not belong to the same set of $\cP^S$,
is greater than or equal to $s^*+\sqrt{\kappa s^*}$.
\item $B$ has the upper and lower bounds
\begin{equation}\label{eq:Bbounds}
\frac{S-s^*-\sqrt{S}}{(d_{\max}-1)(s^*+\sqrt{S}+1)+1}
\leq B \leq \sqrt{S}
\end{equation}
\end{enumerate}

We construct this partition $\cP'$ in the following way: Initialize
$\tilde{T}=T$ and pick any degree-1 vertex of $T$ as its root.
Assign to each edge $(i,j)$ of $\tilde{T}$ a
``score'' of 1 if $i$ and $j$ belong to the same set of $\cP^S$, and 0
otherwise. Repeat the following steps for all vertices $i$ of $T$, in
reverse-breadth-first-search order (starting from a vertex $i$ farthest from the
root):
\begin{itemize}
\item Let $\tilde{T}_i$ be the sub-tree of $\tilde{T}$ rooted at $i$ and
consisting of the descendants of $i$ in $\tilde{T}$.
\item If the total score of edges in $\tilde{T}_i$ is at least 
$s^*+\sqrt{\kappa s^*}$, then add the vertices of
$\tilde{T}_i$ as a set $V_b$ to the partition $\cP'$, and remove
$\tilde{T}_i$ (including the edge from $i$ to its parent) from $\tilde{T}$.
\end{itemize}
This terminates when the remaining tree $\tilde{T}$ has total score less than
$s^*+\sqrt{\kappa s^*}$. Take the last set $R$ of $\cP'$ to be the vertices of
this remaining tree.

By construction, each set $V_1,\ldots,V_B,R$ is connected on
$T$, and property 1 above holds.
To verify the bounds in property 2, note that
the total score of the starting tree $\tilde{T}=T$ is $S$, and the total score of the
final tree belongs to the range $[0,s^*+\sqrt{\kappa s^*})$.
Each time we remove a sub-tree
$\tilde{T}_i$, the score of $\tilde{T}$ decreases by at least $s^*+\sqrt{\kappa
s^*}$. We claim that the
score also decreases by at most $(d_{\max}-1)(s^*+\sqrt{\kappa s^*}+1)+1$:
This is because $i$ has at most $d_{\max}-1$ children, and
if $\tilde{T}_i$ has total score $\geq (d_{\max}-1)(s^*+\sqrt{\kappa s^*}+1)$,
then some sub-tree rooted at one of its children $j$ would have total score
$\geq s^*+\sqrt{\kappa s^*}$.
(The additional $+1$ accounts for a possible $+1$ score on
the edge $(i,j)$.) This sub-tree $\tilde{T}_j$ would have been removed
under the above reverse-breadth-first-search ordering, so this is not possible.
Thus, $\tilde{T}_i$ has total score
$<(d_{\max}-1)(s^*+\sqrt{\kappa s^*}+1)$, verifying our claim. 
Then the total number $B$ of sub-trees removed must satisfy
\[\frac{S-(s_*+\sqrt{\kappa s^*})}{(d_{\max}-1)(s^*+\sqrt{\kappa s^*}+1)+1}
\leq B \leq \frac{S}{s^*+\sqrt{\kappa s^*}}.\]
Recalling $S=\kappa s^*$, this implies (\ref{eq:Bbounds}) as desired.

Now let $\cP^*$ be the partition of $\{1,\ldots,p\}$ induced by $\u^*$ over $T$,
and let $\cP$ be the common refinement of $\cP^S$, $\cP^*$, and $\cP'$
constructed above: Each edge of $T$ which connects two different sets of $\cP$
must connect two different sets of at least one of $\cP^S$, $\cP^*$, and
$\cP'$. Then the subspace $K$ associated to $\cP$ contains $K^S$ and $K^*$, and
furthermore
\[|\partial_T \cP| \leq |\partial_T \cP^S|+
|\partial_T \cP^*|+|\partial_T \cP'| \leq S+s^*+B \leq S+s^*
+\sqrt{S}.\]
Here, we have used $|\partial_T \cP'|=B$ because $\cP'$ consists of $B+1$
connected sets over $T$.

For each $b=1,\ldots,B$, recall the set $V_b$ of $\cP'$, and
construct a vector $\v^b \in \R^p$ whose coordinates are
\[(\v^b)_i=\begin{cases} (\u^*)_i & \text{ if } i \in V_b \\
(\bP_K\u)_i & \text{ if } i \notin V_b. \end{cases}\]
That is, $\v^b$ is equal to $\u^*$ on $V_b$ and equal to $\bP_K\u$ outside $V_b$. Then
\begin{equation}\label{eq:comparison}
\|\bP_K\u-\u^*\|_2^2 \geq \sum_{b=1}^B \sum_{i \in V_b}
|(\bP_K\u)_i-(\u^*)_i|^2=\sum_{b=1}^B \|\bP_K\u-\v^b\|_2^2.
\end{equation}

We claim that $\|\nabla_T \v^b\|_0 \leq S$: Indeed, the edges $(i,j)$ of $T$
where $(\v^b)_i \neq (\v^b)_j$ are contained in the union of
$\partial_T \cP^*$, $\partial_T \cP'$, and the edges of $\partial_T \cP^S$
whose endpoints both belong to the complement of $V_b$.
Since $|\partial_T \cP^S|=S$, and of these $S$ edges, at
least $s^*+\sqrt{\kappa s^*}$ have both endpoints in $V_b$ by property 1 of our construction of $\cP'$, this implies
$\|\nabla_T \v^b\|_0 \leq s^*+B+(S-s^*-\sqrt{\kappa s^*}) \leq S$.

Finally, we use this to lower-bound the right side of (\ref{eq:comparison}):
Observe that by construction, $\u^S$ and all of the vectors $\v^b$
for $b=1,\ldots,B$ belong to the subspace $K$ associated to $\cP$.
Note that
\begin{equation}\label{eq:uSoptimality}
\|\u-\v^b\|_2^2 \geq \|\u-\u^S\|_2^2
\end{equation}
by optimality
of $\u^S$ and the condition $\|\nabla_T \v^b\|_0 \leq S$ shown above. So,
applying the Pythagorean identity for the projection $\bP_K$ and its orthogonal projection $\bP_K^\perp$,
\[\|\bP_K\u-\v^b\|_2^2=\|\u-\v^b\|_2^2-\|\bP_K^\perp \u\|_2^2
\geq \|\u-\u^S\|_2^2-\|\bP_K^\perp \u\|_2^2=\|\bP_K\u-\u^S\|_2^2.\]
Applying this to (\ref{eq:comparison}), we get
\[\|\bP_K\u-\u^*\|_2^2 \geq B \cdot \|\bP_K\u-\u^S\|_2^2.\]
Combining this with the lower-bound on $B$ in (\ref{eq:Bbounds}) yields the lemma.
\end{proof}

\section{Proof of Theorem \ref{thm:main}}

We first extend the result of Lemma \ref{lem:key} to address the discretization error in our approximate projection step (\ref{eq:proj}).

\begin{lemma}\label{lem:main_step1}
In the setting of Lemma \ref{lem:key}, suppose that $\u$ and $\u^*$ are as defined in Lemma \ref{lem:key}, but
\begin{equation}\label{eq:uS}
\u^S=\argmin_{\btheta \in \Delta^p:\|\nabla_T \btheta\|_0 \leq S}
\|\u-\btheta\|_2
\end{equation}
where the minimization is over the discrete lattice
$\Delta=(\Delta_{\min},\Delta_{\min}+\delta,\ldots,\Delta_{\max}-\delta,\Delta_{\max})$.
If $[-\|\u\|_\infty,\|\u\|_\infty] \subseteq [\Delta_{\min},\Delta_{\max}]$,
then the result of Lemma \ref{lem:key} still holds,  with (\ref{eq:projcomparison}) replaced by
\begin{equation}\label{eq:projcomparisondiscrete}
\|\bP_K\u-\u^S\|_2^2 \leq \frac{(d_{\max}-1)(s^*+\sqrt{S}+1)+1}
{S-s^*-\sqrt{S}}\,\|\bP_K\u-\u^*\|_2^2+p\delta^2.
\end{equation}
\end{lemma}

\begin{proof}
The proof is the same as Lemma \ref{lem:key}, up until (\ref{eq:uSoptimality}) where we used optimality of $\u^S$: We define $\cP^S$ and construct $\cP$ as in Lemma \ref{lem:key}, using this discrete vector $\u^S$. Now let us denote by $\check{\u}^S$ the
minimizer of (\ref{eq:uS}) over $\R^p$ rather than over $\Delta^p$. Note that we do not necessarily have $\check{\u}^S \in K^S$, i.e.\ $\check{\u}^S$ may have a different gradient-sparsity pattern from $\u^S$. However, since $\|\nabla_T \v^b\|_0 \leq S$, we still have the bound $\|\u-\v^b\|_2^2 \geq \|\u-\check{\u}^S\|_2^2$ in place of (\ref{eq:uSoptimality}), by optimality of $\check{\u}^S$.

Let $\check{\u}_\Delta^S$ be the vector $\check{\u}^S$ with each entry rounded to the closest value in $\Delta$.
Note that the value of $\check{\u}^S$ on each set of its induced
partition over $T$ is the average of the entries of $\u$ over
this set: This implies that $\|\check{\u}^S\|_\infty \leq \|\u\|_\infty$, and
also that the residual $\u-\check{\u}^S$ is orthogonal to
$\check{\u}^S-\check{\u}^S_\Delta$. By the given condition on $\Delta_{\min}$
and $\Delta_{\max}$, we have the entrywise bound
$\|\check{\u}_\Delta^S-\check{\u}^S\|_\infty \leq \delta$ from the rounding.
Then 
\[\|\u-\v^b\|_2^2 \geq \|\u-\check{\u}^S\|_2^2
=\|\u-\check{\u}_\Delta^S\|_2^2-\|\check{\u}_\Delta^S-\check{\u}^S\|_2^2
\geq \|\u-\check{\u}_\Delta^S\|_2^2-p\delta^2.\]
Since $\check{\u}_\Delta^S \in \Delta^p$ also
satisfies $\|\nabla_T \check{\u}_\Delta^S\|_0 \leq S$,
optimality of $\u^S$ implies $\|\u-\check{\u}_\Delta^S\|_2^2 \geq \|\u-\u^S\|_2^2$. Substituting above and continuing the proof as in Lemma \ref{lem:key}, we get the bound
\[\|\bP_K\u-\u^*\|_2^2 \geq B \cdot (\|\bP_K\u-\u^S\|_2^2-p\delta^2),\]
and rearranging and applying the lower-bound for $B$ concludes the proof as before.
\end{proof}

The second step of the proof is carried out by the following lemma, establishing a key property of the gradient mapping following ideas of
Theorem 2.2.7 in \citep{nesterov2013introductory}.

\begin{lemma}\label{lem:main_step2}
Let $(T_1,T_2)$ be two trees on $\{1,\ldots,p\}$. Let $(\cP_1,\cP_2)$ be two
partitions of $\{1,\ldots,p\}$, with associated subspaces $(K_1,K_2)$,
such that $|\partial_{T_1} \cP_1| \leq s$ and
$|\partial_{T_2} \cP_2| \leq s$ for some sparsity level $s>0$.
Let $K=K_1+K_2$, and let $\bP_K$ be the orthogonal projection onto $K$.

Let $\cL$ be a loss function satisfying cRSC and cRSS with respect to
$(T_1,T_2)$, at sparsity level $s$ and with convexity and smoothness constants $\alpha,L>0$. Fix $\btheta_1 \in K_1$ and define
\[\u=\bP_K(\btheta_1-\nabla \cL(\btheta_1)/L),\qquad
\v=\argmin_{\btheta \in K} \cL(\btheta).\]
Then
\begin{enumerate}[(a)]
\item $\|\u-\v\|_2 \leq \sqrt{1-\alpha/L} \cdot \|\btheta_1-\v\|_2$, and
\item $\|\btheta_1-\v\|_2 \leq (2/\alpha) \cdot \|\bP_K \nabla
\cL(\btheta_1)\|_2$.
\end{enumerate}
\end{lemma}
\begin{proof}
Denote
\[\g=\bP_K \nabla \cL(\btheta_1).\]
Since $\btheta_1 \in K$, we have $\u=\btheta_1-\g/L$. Then
\[\|\u-\v\|_2^2=\|\btheta_1-\v-\g/L\|_2^2
=\|\btheta_1-\v\|_2^2+\frac{1}{L^2}\|\g\|_2^2-\frac{2}{L}\langle\g,
\btheta_1-\v\rangle.\]
So part (a) will follow from
\begin{align}\label{eq:main_step2.1}
\langle\g,\btheta_1-\v\rangle \geq 
\frac{1}{2L}\|\g\|_2^2+\frac{\alpha}{2}\|\btheta_1-\v\|_2^2.
\end{align}

To show (\ref{eq:main_step2.1}),
observe that $\v \in K=K_1+K_2$, so we may apply the cRSC condition
to $\btheta_1$ and $\v$. This gives
\begin{align}\label{eq:RSCapplication}
\cL(\v)\geq \cL(\btheta_1)+\langle \nabla \cL(\btheta_1),\v-\btheta_1\rangle
+\frac{\alpha}{2}\|\v-\btheta_1\|_2^2.
\end{align}
Then, introducing
\[Q(\btheta)=\cL(\btheta_1)+\langle \nabla \cL(\btheta_1),\btheta-\btheta_1
\rangle+\frac{L}{2} \|\btheta-\btheta_1\|_2^2,\]
we get
\[\cL(\v) \geq Q(\u)-\frac{L}{2} \|\u-\btheta_1\|_2^2
+\langle \nabla \cL(\btheta_1),\v-\u \rangle
+\frac{\alpha}{2}\|\v-\btheta_1\|_2^2.\]
Applying $\u-\btheta_1=-\g/L$ and $\v-\u \in K$, this gives
\begin{align*}
\cL(\v) &\geq Q(\u)-\frac{1}{2L} \|\g\|_2^2
+\langle \g,\v-\u \rangle
+\frac{\alpha}{2}\|\v-\btheta_1\|_2^2\\
&= Q(\u)+\frac{1}{2L} \|\g\|_2^2
+\langle \g,\v-\btheta_1 \rangle+\frac{\alpha}{2}\|\v-\btheta_1\|_2^2.
\end{align*}

Next, observe that $\u \in K=K_1+K_2$, so we may apply the cRSS condition to
$\btheta_1$ and $\u$. This yields $\cL(\u) \leq Q(\u)$. Since $\cL(\v) \leq
\cL(\u)$ by optimality of $\v$, combining these observations gives
\[0\geq\frac{1}{2L}\|\g\|_2^2+\langle\g,\v-\btheta_1\rangle+\frac{\alpha}{2}\|\v-\btheta_1\|_2^2.\]
Rearranging yields (\ref{eq:main_step2.1}), which establishes part (a).

For part (b), let us again apply (\ref{eq:RSCapplication}) and
the optimality condition $\cL(\v) \leq \cL(\btheta_1)$ to get
\begin{align*}
0 &\geq \langle \nabla \cL(\btheta_1),\v-\btheta_1\rangle
+\frac{\alpha}{2}\|\v-\btheta_1\|_2^2\\
&=\langle \g,\v-\btheta_1\rangle+\frac{\alpha}{2}\|\v-\btheta_1\|_2^2\\
&\geq -\|\g\|_2 \cdot \|\v-\btheta_1\|_2+\frac{\alpha}{2}\|\v-\btheta_1\|_2^2.
\end{align*}
Rearranging yields part (b).
\end{proof}

\begin{proof}[Proof of Theorem \ref{thm:main}]
Let $\u_t=\btheta_{t-1}-\tfrac{1}{L}\nabla \cL(\btheta_{t-1};Z_1^n)$.
We claim by induction that
\begin{equation}\label{eq:ubound}
[-\|\u_t\|_\infty,\|\u_t\|_\infty] \subseteq [\Delta_{\min},\Delta_{\max}]
\end{equation}
and
\begin{equation}\label{eq:main1}
\|\btheta_t-\btheta^*\|_2 \leq
\Gamma\cdot\|\btheta_{t-1}-\btheta^*\|_2+\frac{4(1+\gamma)}{\alpha}\cdot\Phi(S')+\delta\sqrt{p}
\end{equation}
for each $t=1,\ldots,\tau$.

To start the induction, first observe that for every $t \in \{1,\ldots,\tau\}$, the following holds: Fix any $i \in \{1,\ldots,p\}$ and let $K=K_{t-1}+K^*+\operatorname{span}(\e_i)$ where $(K_{t-1},K^*)$ are the subspaces associated to the partitions induced by $(\btheta_{t-1},\btheta^*)$ over $T_{t-1}$, and $\operatorname{span}(\e_i)$ is the 1-dimensional span of the $i^\text{th}$ standard basis vector $\e_i$.
If $\cP$ is the partition associated to $K$, then $|\partial_{T_{t-1}} \cP| \leq S+2s^*+d_{\max} \leq S'$ because $\|\nabla_{T_{t-1}} \btheta_{t-1}\|_0 \leq S$,
$\|\nabla_{T_{t-1}} \btheta^*\|_0 \leq 2s^*$ by Lemma \ref{lem:tree}, and $\|\nabla_{T_{t-1}} \e_i\|_0 \leq d_{\max}$. Applying the cRSS property for $\cL$ with respect to $(T_{t-1},T_t)$, we get that
the loss $\cL(\cdot\,;Z_1^n)$ is $L$-strongly-smooth restricted to $K$, meaning for all $\u,\v \in K$,
\[\cL(\u;Z_1^n) \leq \cL(\v;Z_1^n)+\langle \nabla \cL(\v),\u-\v \rangle
+\frac{L}{2}\|\u-\v\|_2^2.\]
Then applying Eq.\ (2.1.8) of \citep{nesterov2013introductory} to the loss $\cL(\cdot\,;Z_1^n)$ restricted to $K$, we have for all $\u,\v \in K$ that
\[\|\bP_K \nabla \cL(\u;Z_1^n)-\bP_K \nabla \cL(\v;Z_1^n)\|_2
\leq L\|\u-\v\|_2,\]
where $\bP_K$ is the orthogonal projection onto $K$. In particular,
\[\big|\langle \e_i,\nabla \cL(\btheta_{t-1};Z_1^n)-\nabla \cL(\btheta^*;Z_1^n) \rangle\big| \leq L\|\btheta_{t-1}-\btheta^*\|_2.\]
This holds for each standard basis vector $\e_i$, so
\begin{equation}\label{eq:gradientinfbound}
\frac{1}{L}\|\nabla \cL(\btheta_{t-1};Z_1^n)\|_\infty \leq \frac{1}{L}\|\nabla \cL(\btheta^*;Z_1^n)\|_\infty
+\|\btheta_{t-1}-\btheta^*\|_2.
\end{equation}
Then (\ref{eq:ubound}) holds for $t=1$ by the initialization $\btheta_0=0$ and the given conditions for $\Delta_{\min},\Delta_{\max}$.

Suppose by induction that (\ref{eq:ubound}) holds for $t$.
We apply Lemma \ref{lem:main_step1} to $T=T_t$, $\u^*=\btheta^*$, and
$\u=\u_t$. Note that by Lemma \ref{lem:tree}, $\|\nabla_T \btheta^*\|_0 \leq
2s^*$. Then by the definition of the update (\ref{eq:proj}), we have
$\u^S=\btheta_t$ in Lemma \ref{lem:main_step1}.
Denote by $\cP_2$ the partition guaranteed by
Lemma \ref{lem:main_step1}, with associated subspace $K_2$.
Then the lemma guarantees that
\[|\partial_{T_t} \cP_2| \leq S+2s^*+\sqrt{S} \leq S',\]
and furthermore
\[\|\bP_{K_2}\u_t-\btheta_t\|_2 \leq \gamma \cdot \|\bP_{K_2}\u_t-\btheta^*\|_2+\delta\sqrt{p}.\]
This bound implies
\begin{equation}\label{eq:step1}
\|\btheta_t-\btheta^*\|_2 \leq
\|\btheta_t-\bP_{K_2}\u_t\|_2+\|\bP_{K_2}\u_t-\btheta^*\|_2
\leq (1+\gamma)\|\bP_{K_2}\u_t-\btheta^*\|_2+\delta\sqrt{p}.
\end{equation}

Next, let us apply Lemma \ref{lem:main_step2}: Take $(T_1,T_2)$ in Lemma \ref{lem:main_step2} to be
$(T_{t-1},T_t)$. Take $\cP_1$ to be the common refinement of the partitions
induced by $\btheta_{t-1}$ and $\btheta^*$ over $T_{t-1}$, and let $\cP_2$ be as
above. Then $|\partial_{T_{t-1}}\cP_1| \leq S+2s^*<S'$ and
$|\partial_{T_t}\cP_2| \leq S'$, so the cRSC and cRSS conditions
required in Lemma \ref{lem:main_step2} are satisfied.
Let $K_1,K_2$ be the associated subspaces, and set $K=K_1+K_2$ and
\[\v=\argmin_{\btheta \in K} \cL(\btheta;Z_1^n).\]
First, we take $\btheta_1$ to be $\btheta_{t-1}$, and apply
Lemma \ref{lem:main_step2}(a) with $\u=\bP_K \u_t$. This gives
\begin{equation}\label{eq:step2}
\|\bP_K\u_t-\v\|_2 \leq
\sqrt{1-\frac{\alpha}{L}} \cdot \|\btheta_{t-1}-\v\|_2.
\end{equation}
Note that $\|\bP_{K_2}\u_t-\btheta^*\|_2 \leq \|\bP_K\u_t-\btheta^*\|_2$ 
because $\btheta^* \in K_2 \subseteq K$.
Applying this and (\ref{eq:step2}) to (\ref{eq:step1}),
\begin{align}
\|\btheta_t-\btheta^*\|_2 &\leq
(1+\gamma)\|\bP_K\u_t-\btheta^*\|_2+\delta\sqrt{p}\nonumber\\
&\leq (1+\gamma)\left(
\sqrt{1-\frac{\alpha}{L}}\cdot \|\btheta_{t-1}-\v\|_2+\|\v-\btheta^*\|_2\right)
+\delta\sqrt{p}\nonumber\\
&\leq (1+\gamma)\left(\sqrt{1-\frac{\alpha}{L}} \cdot
\|\btheta_{t-1}-\btheta_*\|_2+2\|\v-\btheta_*\|_2\right)+\delta\sqrt{p}.\label{eq:tmpbound}
\end{align}
Now, let us apply Lemma \ref{lem:main_step2}(b) with
$\btheta_1$ being $\btheta^*$. This gives
\[\|\v-\btheta_*\|_2 \leq (2/\alpha)\|\bP_K \nabla \cL(\btheta^*;Z_1^n)\|_2
\leq (2/\alpha)\Phi(S'),\]
the second bound holding by the cPGB assumption. Applying this to (\ref{eq:tmpbound}) establishes (\ref{eq:main1}) at the iterate $t$.

We may apply (\ref{eq:main1}) recursively for $1,\ldots,t$, using
$\btheta_0=0$ and $1+\Gamma+\Gamma^2+\ldots=1/(1-\Gamma)$, to get
\begin{equation}\label{eq:inductivebound}
\|\btheta_t-\btheta^*\|_2 \leq \Gamma^t \cdot \|\btheta^*\|_2
+\frac{1}{1-\Gamma}\left(\frac{4(1+\gamma)}{\alpha}\cdot\Phi(S')
+\delta\sqrt{p}\right)=\Gamma^t \cdot \|\btheta^*\|_2+\Lambda.
\end{equation}
In particular,
\[\|\btheta_t\|_2 \leq 2\|\btheta^*\|_2+\Lambda.\]
Then, applying also (\ref{eq:gradientinfbound}),
\begin{align*}
\|\u_{t+1}\|_\infty &\leq \|\btheta_t\|_\infty
+\frac{1}{L}\|\nabla \cL(\btheta^*;Z_1^n)\|_\infty+\|\btheta_t-\btheta^*\|_\infty\\
&\leq \frac{1}{L}\|\nabla \cL(\btheta^*;Z_1^n)\|_\infty
+3\|\btheta^*\|_2+2\Lambda.
\end{align*}
Then the given condition for $\Delta_{\min},\Delta_{\max}$ implies that
(\ref{eq:ubound}) holds for iteration $t+1$, completing the
induction. Finally, the theorem follows by applying (\ref{eq:inductivebound}) at $t=\tau$.
\end{proof}

\section{Proofs for cRSC, cRSS, and cPGB}

\begin{proof}[Proof of Lemma \ref{lem:ProjGrad}]
Note that there are $\binom{p-1}{S}$ different partitions $\cP_1$ of $V=\{1,\ldots,p\}$ with $|\partial_{T_1}\cP_1|=S$, and similarly for $\cP_2$, because each such partition corresponds to cutting $S$ of the $p-1$ edges of $T_1$. Let $g(S)=S\log(1+p/S)$. Then there are at most $\binom{p-1}{S}\cdot\binom{p-1}{S}\leq e^{2g(S)}$ different combinations of $(K_1,K_2)$, and hence at most this many subspaces $K$. Taking a union bound over all such $K$ gives, for any $\zeta>0$,
\begin{align*}
\P(\max_{K}\|\bP_K\nabla \cL(\btheta^*;Z_1^n)\|_2\geq \zeta)\leq e^{2g(S)}\cdot\max_{K}\P(\|\bP_K\nabla \cL(\btheta^*;Z_1^n)\|_2\geq \zeta).
\end{align*}
Note that the dimension of $K$ is less than the sum of dimensions of $K_1$ and $K_2$, which is at most $2(S+1)$. Applying a covering net argument, we may find a $1/2$-net $\N_{1/2}$ for the set $\{\v\in K:\|\v\|_2=1\}$ of cardinality at most $5^{2S+2}$. Thus,
\begin{align*}
\P(\|\bP_K\nabla \cL(\btheta^*;Z_1^n)\|_2\geq \zeta)&\leq \P(2\max_{\v\in \N_{1/2}}|\v^\top\nabla \cL(\btheta^*;Z_1^n)|\geq \zeta)\\
&\leq 5^{2S+2}\cdot\max_{\v\in \N_{1/2}} \P(2|\v^\top\nabla \cL(\btheta^*;Z_1^n)|\geq \zeta).
\end{align*}
Applying the subgaussian assumption on $\v^\top \nabla\cL(\btheta^*;Z_1^n)$, we get
\begin{align*}
\P(\max_{K}\|\bP_K\nabla \cL(\btheta^*;Z_1^n)\|_2\geq \zeta)\leq e^{2g(S)}\cdot 5^{2S+2}\cdot 2e^{-n\zeta^2/8\sigma^2}.
\end{align*}
Then for any $k>0$ and some constant $C_k>0$ depending only on $k$, setting $\zeta=\sqrt{C_k\sigma^2 g(S)/n}$ and applying $g(S)\geq\log(1+p)$, we get
\begin{align*}
\P(\max_{K}\|\bP_K\nabla \cL(\btheta^*;Z_1^n)\|_2\geq \sqrt{C_k\sigma^2 g(S)/n})\leq p^{-k}.
\end{align*}
\end{proof}

\begin{proof}[Proof of Proposition \ref{thm:prop1}]
We will consider a fixed $t$, and then apply a union bound over $1\leq t\leq \tau$.

For cRSC and cRSS, note that $\cL(\btheta;Z_1^n)=\frac{1}{2n}\|\y-\X\btheta\|_2^2$ for the linear model, which gives $\cL(\btheta_2;Z_1^n)-\cL(\btheta_1;Z_1^n)-\langle \btheta_2-\btheta_1,\nabla\cL(\btheta_1;Z_1^n)\rangle=\frac{1}{2n}\|\X(\btheta_1-\btheta_2)\|_2^2$. Then the cRSC and cRSS bounds will hold as long as
\begin{align}\label{eq:prop1.1}
\sup_K \sup_{\u\in K:\|\u\|_2=1}\frac{1}{n}\|\X\u\|_2^2\leq 3\lambda_1/2~~~{\rm and}~~~\inf_K \inf_{\u\in K:\|\u\|_2=1}\frac{1}{n}\|\X\u\|_2^2\geq \lambda_p/2,
\end{align}
where the supremum and infimum are over all subspaces $K=K_1+K_2$ as in Definition \ref{def:cRSC_cRSS}. This property (\ref{eq:prop1.1}) is invariant under a common rescaling of $\X^\top\X$, $\lambda_1$, and $\lambda_p$, so we may assume that $\lambda_p=1$.

Fixing any such subspace $K$,
note that the dimension of $K$ is upper bounded by $2S'+2$. Let $\bP_K$ be the orthogonal projection onto $K$, and write $\bP_K=\Q_K\Q_K^\top$, where $\Q_K$ has orthonormal columns spanning $K$. Then $\X\Q_K$ also has independent rows $\x_i^\top\Q_K$, where $\|\Q_K^\top\x_i\|_{\psi_2}^2\leq D$ and $\Cov[\Q_K^\top\x_i]=\Q_K^{\T}\bSigma\Q_K$.  Applying Eq.\ (5.25) of \citep{vershynin2010introduction} to $\X\Q_K$, for any $\zeta>0$ and some constants $C_3,C_4>0$ depending only on $D$,
\begin{align*}
\P\left[\left\|\frac{1}{n}\Q_K^\top\X^\top\X\Q_K-\Q_K^\top\bSigma\Q_K\right\|_{\op}\geq \max(\omega,\omega^2)\right]\leq 2e^{-C_3\zeta^2},\quad\quad \omega\equiv\frac{C_4\sqrt{S'}+\zeta}{\sqrt{n}}.
\end{align*}
Recall $g(S')=S'\log(1+\tfrac{p}{S'})$. Note that there are at most $\binom{p-1}{S'}\cdot\binom{p-1}{S'}\leq  e^{2g(S')}$ different subspaces $K$. Taking a union bound over $K$, and noting that any $\u\in K$ may be represented as $\u=\Q_K\v$ for such $K$, this yields
\begin{align*}
\P\left[\sup_{K}\sup_{\u\in K:\|\u\|_2=1}\left|\frac{1}{n}\u^{\T}\X^{\T}\X\u-\u^{\T}\bSigma\u\right|\geq \max(\omega,\omega^2)\right]\leq 2 e^{2g(S')-C_3\zeta^2}
\end{align*}
When $\|\u\|_2=1$, $\u^{\T}\bSigma\u\in [\lambda_p,\lambda_1]$. It follows, with probability at least $1-2 e^{2g(S')-C_3\zeta^2}$ and under our scaling $\lambda_p=1$, that
\begin{align*}
\sup_K \sup_{\u\in K:\|\u\|_2=1} \frac{1}{n}\|\X\u\|_2^2\leq \lambda_1+\max(\omega,\omega^2),
\end{align*}
and 
\begin{align*}
\inf_K \inf_{\u\in K:\|\u\|_2=1} \frac{1}{n}\|\X\u\|_2^2\geq (1-\max(\omega,\omega^2))_+.
\end{align*}
Then, for any $k>0$ and some constants $C_1,C_5>0$ depending only on $k,D$, assuming $n\geq C_1g(S')$ and setting $\zeta=\sqrt{C_5g(S')}$, (\ref{eq:prop1.1}) holds with probability at least $1-2e^{-kg(S')}$. Applying $g(S')\geq \log p$, this probability is at least $1-2p^{-k}$.

For cPGB, it follows from the first part of the proof that with probability at least $1-2p^{-k}$, $\|\X\u\|_2^2/n^2\leq 3\lambda_1/2n$ for every such subspace $K$ and every $\u\in K$. Applying Lemma 5.9 of \citep{vershynin2010introduction} and the assumption $\|e_i\|_{\psi_2}^2\leq \sigma^2$, conditional on $\X$ and this event, $\u^{\T}\X^\top\e/n$ is a subgaussian random variable with subgaussian parameter $C_6\lambda_1\sigma^2/n$, where $C_6>0$ is some absolute constant. Noting that $\nabla\cL(\btheta^*;Z_1^n)=-\X^\top\e/n$ and applying Lemma \ref{lem:ProjGrad}, $\cL$ has the cPGB $\Phi(S')=C_2\sigma\sqrt{\lambda_1g(S')/n}$ with probability at least $1-3p^{-k}$.

The bound for $\|\nabla \cL(\btheta^*;Z_1^n)\|_\infty=\|\X^\top \e/n\|_\infty$ follows from similarly noting that with probability at least $1-2p^{-k}$, $\|\X\u\|_2^2/n^2 \leq 3\lambda_1/2n$ for each standard basis vector $\u \in \R^p$. Conditional on $\X$ and this event, $\u^{\T}\X^\top\e/n$ is subgaussian with parameter $C_6\lambda_1\sigma^2/n$ for every standard basis vector $\u$. Then the bound for $\|\X^\top \e/n\|_\infty$ follows from the subgaussian tail bound and a union bound over all such $\u$. Finally, applying a union bound over $1\leq t\leq \tau$ completes the proof.
\end{proof}

\begin{proof}[Proof of Proposition \ref{thm:prop2}]
Similar to the proof of Proposition \ref{thm:prop1}, we consider fixed $t$ and then apply a union bound over $1\leq t\leq \tau$.

For cRSC and cRSS, note that $\cL(\btheta;Z_1^n)=\frac{1}{n}\sum_{i=1}^n (b(\x_i^\top\btheta)-y_i\x_i^\top \btheta)$, which gives
\begin{align*}
&\cL(\btheta_2;Z_1^n)-\cL(\btheta_1;Z_1^n)-\langle \btheta_2-\btheta_1,\nabla\cL(\btheta_1;Z_1^n)\rangle\\
&=\frac{1}{n}\sum_{i=1}^n (b(\x_i^{\T}\btheta_2)-b(\x_i^{\T}\btheta_1)-b'(\x_i^{\T}\btheta_1)x_i^{\T}(\btheta_2-\btheta_1)).
\end{align*}
Applying the assumption on $b$,
\begin{align*}
\frac{\alpha_b}{2n}\|\X(\btheta_2-\btheta_1)\|_2^2\leq\cL(\btheta_2;Z_1^n)-\cL(\btheta_1;Z_1^n)-\langle \btheta_2-\btheta_1,\nabla\cL(\btheta_1;Z_1^n)\rangle\leq  \frac{L_b}{2n}\|\X(\btheta_2-\btheta_1)\|_2^2.
\end{align*}
Then cRSC and cRSS hold for $(T_{t-1},T_t)$ with probability $1-2p^{-k}$, by (\ref{eq:prop1.1}) and the same argument as Proposition \ref{thm:prop1}.

For cPGB, note that $\nabla \cL(\btheta^*;Z_1^n)=-\frac{1}{n}\sum_{i=1}^n \x_i e_i=-\X^\T\e/n$ where $\e=(e_1,\ldots,e_n)$. Similar to the proof of Proposition \ref{thm:prop1}, we condition on $\X$ and the probability $1-2e^{-kg(S')}$ event $\mathcal{E}$ that $\frac{1}{n}\|\X\u\|_2^2\leq 3\lambda_1/2$ for every $K=K_1+K_2$ and every $\u\in K$. Then similar to the proof of Lemma \ref{lem:ProjGrad}, we get for any $\zeta>0$
\begin{align*}
&\P(\sup_{K}\|\bP_K\X^{\T}\e\|_2/\sqrt{n}>\zeta)\\
&\leq  e^{2g(S')}\cdot 5^{2S'+2}\cdot \Big(\sup_{\w:\|\w\|_2=1}\P(\{2|\w^\T\X^\T\e|/\sqrt{n} \geq \zeta\} \cap \mathcal{E})+2e^{-kg(S')}\Big).
\end{align*}
Note that (\ref{eq:GLMecondition}) implies $\Var(e_i)\leq C_3$ where $C_3>0$ is some constant depending only on $D_1,D_2,\beta$. If $1<\beta\leq 2$, applying Lemma \ref{lem:aux1}, 
\begin{align*}
\P(\sup_{K}\|\bP_K\X^{\T}\e\|_2/\sqrt{n}>\zeta)\leq  e^{2g(S')}\cdot 5^{2S'+2}\cdot \Big(2e^{-\zeta^\beta/(C_4\sqrt{\lambda_1})^\beta}+2e^{-kg(S')}\Big),
\end{align*}
where $C_4>0$ is some constant depending only on $D_1,D_2,\beta$. Then for any $k>0$ and some constant $C_2>0$ depending only on $k,D,D_1,D_2,\beta$, setting $\zeta=C_2\sqrt{\lambda_1}\cdot g(S')^{1/\beta}$ and applying $g(S')\geq \log p$, we have
\begin{align*}
\P(\sup_{K}\|\bP_K\X^{\T}\e\|_2/n>C_2\sqrt{\lambda_1/n}\cdot g(S')^{1/\beta})\leq  p^{-k}.
\end{align*}
If $\beta=1$, applying Lemma \ref{lem:aux1}, we get 
\begin{align*}
\P(\sup_{K}\|\bP_K\X^{\T}\e\|_2/n>C_2\sqrt{\lambda_1/n}\log n\cdot g(S'))\leq  p^{-k}.
\end{align*}

The bound for $\|\nabla \cL(\btheta^*;Z_1^n)\|_\infty=\|\X^\top \e/n\|_\infty$ is similar to the proof of Proposition \ref{thm:prop1}. Note that with probability at least $1-2p^{-k}$, $\|\X\u_i\|_2^2/n\leq 3\lambda_1/2$ for each standard basis vector $\u_i\in\R^p$ with $1\leq i\leq p$. We condition on $\X$ and this event $\mathcal{E}'$ and get for any $\zeta>0$
\begin{align*}
\P(\max_{1\leq i\leq p}|\u_i\X^{\T}\e|/\sqrt{n}>\zeta)\leq p\cdot\Big(\max_{1\leq i\leq p}\P(\{|\u_i\X^{\T}\e|/\sqrt{n}>\zeta\}\cap \mathcal{E}')+2p^{-k}\Big).
\end{align*}
Similarly, if $1<\beta\leq 2$, applying Lemma \ref{lem:aux1}, for any $k>0$ and some constant $C_3$ depending only on $k,D,D_1,D_2,\beta$, we get
\begin{align*}
\P(\max_{1\leq i\leq p}|\u_i\X^{\T}\e|/n> C_3(\log p)^{1/\beta}\sqrt{\lambda_1/n})\leq p^{-k}.
\end{align*}
If $\beta=1$, applying Lemma \ref{lem:aux1}, we get 
\begin{align*}
\P(\max_{1\leq i\leq p}|\u_i\X^{\T}\e|/n> C_3(\log n)(\log p)\sqrt{\lambda_1/n})\leq p^{-k}.
\end{align*}
Finally, applying the union bound over $1\leq t\leq \tau$ completes the proof.
\end{proof}

\section{Auxilliary Lemma}
The following lemma comes from \citep[Lemma 1]{huang2008adaptive}.

\begin{lemma}\label{lem:aux1}
Suppose $X_1,\ldots,X_n$ are $i.i.d.$ random variables with $\E X_i=0$ and $\Var (X_i)=\sigma^2$. Further suppose, for $1\leq d\leq 2$ and certain constants $C_1,C_2>0$, their tail probabilities satisfy 
\begin{align*}
\P(|X_i|\geq\zeta)\leq C_1\exp(-C_2\zeta^d),
\end{align*}
for all $\zeta>0$. Let $c_1,\ldots,c_n$ be constants satisfying $\sum_{i=1}^n c_i\leq M^2$ and $W=\sum_{i=1}^n c_iX_i$. Then we have
\begin{align*}
\|W\|_{\psi_d}\leq\left\{\begin{array}{cc}
K_dM \{\sigma+C_3\},&~~~1<d\leq 2\\
K_1M \{\sigma+C_4\log n\},&~~~d=1
\end{array}\right.
\end{align*}
where $K_d$ is a positive constant depending only on $d$, $C_3$ is some positive constant depending only on $C_1,C_2,d$ and $C_4$ is some positive constant depending only on $C_1,C_2$. Consequently,
\begin{align*}
\P(|W|>\zeta)\leq\left\{\begin{array}{cc}
2\exp\{-(\zeta/(K_dM(\sigma+C_3)))^d\},&~~~1<d\leq 2,\\
2\exp\{-\zeta/(K_1M(\sigma+C_4\log n))\},&~~~d=1.
\end{array}\right.
\end{align*}
\end{lemma}

\end{document}